\documentclass{article} 
\usepackage{iclr27,times}

\iclrfinalcopy

\usepackage{amsmath,amsfonts,bm}

\def\1{\bm{1}}

\DeclareMathAlphabet{\mathsfit}{\encodingdefault}{\sfdefault}{m}{sl}
\SetMathAlphabet{\mathsfit}{bold}{\encodingdefault}{\sfdefault}{bx}{n}

\usepackage{hyperref}
\hypersetup{linktoc=all}
\usepackage{url}

\usepackage{amsmath}
\usepackage{amssymb}
\usepackage{multirow}
\usepackage{booktabs}
\usepackage{graphicx}
\usepackage{makecell}
\usepackage{pifont} 
\usepackage{subcaption}
\usepackage{paralist}
\newcommand{\cmark}{\ding{51}} 
\newcommand{\xmark}{\ding{55}} 
\newcommand{\circnum}[1]{\ding{\numexpr181+#1\relax}}
\usepackage{array}
\usepackage{amsthm}
\usepackage{mathtools}
\usepackage{algorithm}
\usepackage{algorithmic}
\usepackage{wrapfig}
\usepackage{needspace}
\usepackage{tabularx}
\usepackage{titletoc}
\providecommand{\coloneqq}{\mathrel{\mathop:}=}

\theoremstyle{definition}
\newtheorem{definition}{Definition}[section]

\theoremstyle{plain}

\newcolumntype{L}[1]{>{\raggedright\arraybackslash}p{#1}}

\newcommand{\tablestyle}{
    \small
    \setlength{\tabcolsep}{4.6pt}
    \renewcommand{\arraystretch}{1.0}
}

\newcommand{\appendixorganization}{%
    \begin{center}
        {\Large\bfseries Appendix Contents}
    \end{center}
    \vspace{6pt}
    \noindent\textbf{Organization.}
    In this appendix, we provide further details as follows:
    \par\smallskip
    \begingroup
    \setcounter{tocdepth}{2}
    \titlecontents{section}
        [2.8em]
        {\small\bfseries}
        {\hyperlink{appendix.\thecontentslabel}{\contentslabel{2.8em}}}
        {}
        {\titlerule*[0.55pc]{.}\contentspage}
        [\vspace{6pt}]
    \titlecontents{subsection}
        [4.3em]
        {\small}
        {\hyperlink{subsection.\thecontentslabel}{\contentslabel{2.8em}}}
        {}
        {\titlerule*[0.55pc]{.}\contentspage}
        [\vspace{6pt}]
    \hrule height 0.4pt
    \vspace{4pt}
    \printcontents[appendix]{}{1}{}
    \vspace{1pt}
    \hrule height 0.4pt
    \endgroup
    \vspace{6pt}
}

\title{\textsc{BitsMoE}: Cost-Aware Bit Allocation in Spectral Space for MoE LLM Quantization}

\author{%
    \begingroup
        \newcommand{\namefont}{\bfseries\fontsize{10.5}{12}\selectfont}
        \newcommand{\affilfont}{\normalfont\mdseries\fontsize{10}{12}\selectfont}
        \newcommand{\contactfont}{\normalfont\mdseries\fontsize{10}{11}\selectfont}
        \hspace*{-7pt}\begin{tabular}{@{}l@{}}
            {\namefont Jiayu Zhao$^{1,2}$\thanks{Work done during a visit to Nanyang Technological University.}%
            \thanks{\raggedright Correspondence to: \texttt{jiayz@mail.ustc.edu.cn} and \texttt{liu@ntu.edu.sg}.}%
            \quad Zihan Teng$^{1,2}$ \quad
            Minhao Fan$^{2}$ \quad Tianrui Ma$^{2}$} \\
            {\namefont Wentao Ren$^{3}$ \quad Song Chen$^{1}$ \quad
            Weichen Liu$^{2}$} \\[3pt]
            {\affilfont $^{1}$School of Microelectronics, University of Science and Technology of China} \\
            {\affilfont $^{2}$College of Computing and Data Science, Nanyang Technological University} \\
            {\affilfont $^{3}$School of Electrical and Electronic Engineering, Nanyang Technological University} \\[3pt]
            {\contactfont \url{https://github.com/zjiayu064/BitsMoE}}
        \end{tabular}
    \endgroup
}

\begin{document}

\maketitle

\lhead{Preprint. Under review.}

\begin{abstract}
Mixture-of-Experts~(MoE) large language models incur substantial memory costs due to their large expert parameter counts. Mixed-precision quantization reduces these costs by allocating different bit-widths to experts or linear blocks according to their importance. However, assigning a single precision within each expert or linear block overlooks its internal structural heterogeneity. This limitation motivates two key questions: \textbf{(1)} how to define a fine-grained unit for quantization within a linear transformation; and \textbf{(2)} how to characterize the quantization cost of each unit under actual activation patterns and different bit-widths. To address these two questions, we propose \textsc{BitsMoE}, a cost-aware mixed-precision quantization framework built on two complementary techniques: \textbf{(1)} Shared-basis Spectral Decomposition~(SSD) separates expert weights into a shared basis and expert-specific spectral components, defining structural quantization units while exploiting cross-expert redundancy. \textbf{(2)} Factorized Quantization Cost Modeling~(FQCM) estimates component-wise costs from output reconstruction loss by combining intrinsic spectral importance, activation-dependent importance, and bit-width-dependent distortion. Using these component-wise costs, we formulate bit allocation as an integer linear program~(ILP) that minimizes total modeled quantization cost under a fixed memory budget. On Qwen3-30B-A3B at 2-bit, \textsc{BitsMoE} achieves 64.29\% average accuracy over seven downstream tasks, outperforming the evaluated MoE-specific methods, including those using ILP-based bit allocation, and exceeding GEMQ by 2.80 percentage points. Under the same setting, it achieves a $16.47\times$ end-to-end offline quantization speedup over GEMQ. It also achieves up to $6.46\times$ the decode throughput of GPTQ.
\end{abstract}

\section{Introduction}
\label{sec:intro}
Mixture-of-Experts~(MoE)~\citep{cai2025survey} has emerged as one of the most effective paradigms for scaling large language models~(LLMs). By leveraging dynamic routing mechanisms and scalable capabilities, MoE models have demonstrated superior performance and achieved state-of-the-art results~\citep{jiang2024mixtral, deepseekai2024deepseekv3technicalreport}. However, the growth in expert count substantially increases parameter storage and memory bandwidth requirements. For instance, Qwen3-Next-80B-A3B~\citep{qwen3next} activates only 3B parameters per token but has an FP16 GPU memory footprint of approximately 149\,GB. These extreme resource requirements make deployment on memory-constrained GPUs largely infeasible.

To address these costs, post-training quantization~(PTQ) has emerged as a promising solution for LLMs without the need for retraining~\citep{lin2024awq, xiao2023smoothquant, shao2024omniquant, ashkboos2024quarot}. Typical PTQ methods treat all weight matrices uniformly, without accounting for the fact that MoE experts differ in their functional roles and contributions to model performance. Recent MoE-specific quantization methods~\citep{duanmu2025mxmoe, huang2025milo} recognize such expert-level heterogeneity and tailor quantization configurations at the expert or linear-block level, thereby mitigating performance degradation at low bit-widths. However, assigning a single quantization configuration to an expert or linear block overlooks its intrinsic structural heterogeneity. Specifically, the quantization sensitivity of a linear transformation is inherently non-uniform across directions and further varies with the activation patterns presented to it.

These observations raise two key issues:
\circnum{1} \textbf{What should be the fine-grained unit for quantization within a linear transformation?} Simply partitioning a weight matrix along channels or dimensions offers finer granularity but lacks structural meaning. Moreover, independent partitioning fails to exploit the substantial redundancy among experts within the same MoE layer~\citep{liu2025efficient, li2025moesvd}. An effective quantization unit should therefore capture both the intrinsic structure and the shared structure across experts.
\circnum{2} \textbf{How should the quantization cost of each unit be characterized under actual activation patterns and different bit-widths?} Quantization cost depends not only on actual activation patterns, but also on the intrinsic sensitivity of each unit and the perturbation introduced by the quantizer. These factors jointly determine the impact of assigning a particular bit-width. A meaningful cost model should therefore couple intrinsic structure, activation-dependent importance, and quantization distortion to guide mixed-precision allocation.

MoE-SVD~\citep{li2025moesvd} reduces redundancy among experts by sharing an SVD-derived basis. Delta-CoMe~\citep{ping2024delta} uses singular vectors as quantization units and assigns them different bit-widths to compress delta weights. These ideas provide a starting point for the problem~\circnum{1}. However, a fixed bit-allocation strategy based solely on singular-value ordering does not explicitly account for differences in quantization cost across units. This leaves problem~\circnum{2} unresolved: how to combine intrinsic sensitivity, activation-dependent importance, and bit-width-dependent quantization distortion into a quantization cost model for each unit.

\begin{figure*}[t]
    \centering

    \begin{subfigure}[t]{0.32\textwidth}
        \centering
        \includegraphics[width=\linewidth]{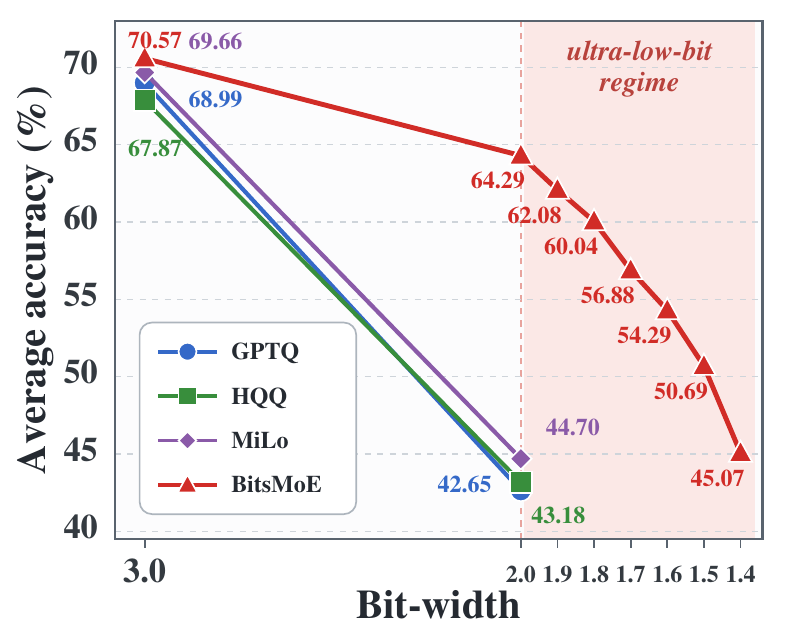}
        \caption{Ultra-low-bit accuracy.}
        \label{fig:motivation_accuracy}
    \end{subfigure}
    \hfill
    \begin{subfigure}[t]{0.32\textwidth}
        \centering
        \includegraphics[width=\linewidth]{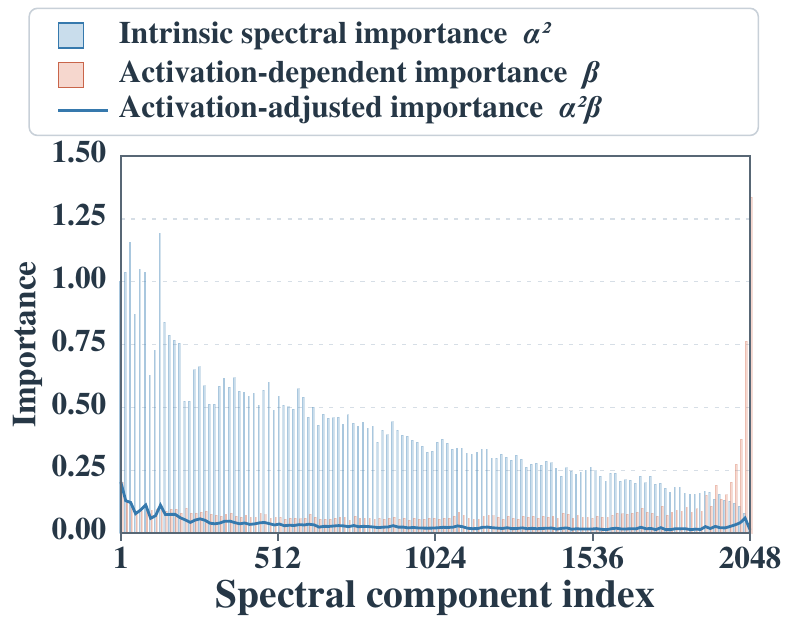}
        \caption{Spectral importance.}
        \label{fig:motivation_importance}
    \end{subfigure}
    \hfill
    \begin{subfigure}[t]{0.32\textwidth}
        \centering
        \includegraphics[width=\linewidth]{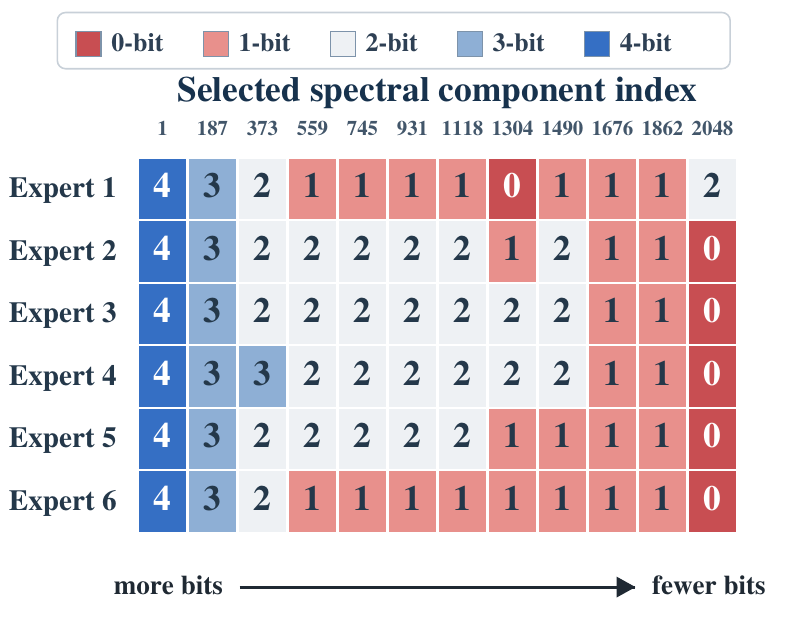}
        \caption{Component-wise allocation.}
        \label{fig:motivation_allocation}
    \end{subfigure}

    \caption{
        Results on Qwen3-30B-A3B.
        (a) Average accuracy over seven downstream tasks under ultra-low-bit quantization.
        (b) Spectral components exhibit highly non-uniform importance.
        (c) \textsc{BitsMoE} assigns heterogeneous bit-widths to individual spectral components under a fixed bit budget.
    }
    \label{fig:motivation}
\end{figure*}

To address these issues, we propose \textsc{BitsMoE}, which is built on two efficient innovations:\circnum{1} \textbf{Shared-basis Spectral Decomposition~(SSD)}. First, we decompose the expert weights within each MoE layer into a shared basis and expert-specific spectral components. The shared basis captures common structure across experts, while each spectral component serves as a fine-grained structural unit for quantization. \circnum{2} \textbf{Factorized Quantization Cost Modeling~(FQCM)}. Second, we model the quantization cost of each component by jointly considering its intrinsic spectral importance, activation-dependent importance, and quantization distortion. Based on this cost model, we formulate bit allocation as an integer linear program~(ILP) that assigns a bit-width to each spectral component under a fixed bit budget.
As illustrated in Figs.~\ref{fig:motivation_importance} and~\ref{fig:motivation_allocation}, spectral components vary substantially in spectral importance, and \textsc{BitsMoE} assigns heterogeneous bit-widths accordingly.
Extensive experiments demonstrate that \textsc{BitsMoE} achieves strong accuracy across a range of bit budgets in the ultra-low-bit regime (Fig.~\ref{fig:motivation_accuracy}).

Our main contributions are summarized as follows:

\begin{itemize}
\item We introduce \textbf{Shared-basis Spectral Decomposition~(SSD)} to define structurally meaningful, fine-grained quantization units within MoE experts. SSD separates shared structure across experts from expert-specific spectral components, capturing cross-expert redundancy while exposing the intrinsic heterogeneity within each expert.

\item We derive \textbf{Factorized Quantization Cost Modeling~(FQCM)} from the output reconstruction loss, decomposing component-wise quantization cost into intrinsic spectral importance, activation-dependent importance, and bit-width-dependent distortion. FQCM estimates the impact of each precision assignment under observed activation patterns, providing a cost criterion beyond singular-value ordering.

\item Building on SSD and FQCM, we formulate MoE quantization as a \textbf{fine-grained bit-budget allocation problem}. The resulting ILP minimizes the total modeled quantization cost under a fixed memory budget, enabling cost-aware precision allocation that selectively preserves critical expert-specific components.

\item Experiments demonstrate the accuracy, offline quantization efficiency, and inference efficiency of \textbf{\textsc{BitsMoE}}. At 2-bit precision on Qwen3-30B-A3B, \textsc{BitsMoE} achieves 64.29\% average accuracy over seven downstream tasks, outperforming the evaluated MoE-specific methods, including those using ILP-based bit allocation, and exceeding GEMQ by 2.80 percentage points. Under the same setting, it achieves a $16.47\times$ end-to-end offline quantization speedup over GEMQ. It also delivers up to $6.46\times$ the decode throughput of GPTQ.
\end{itemize}

\section{Related Work}

\paragraph{Mixture-of-Experts Large Language Models}
MoE architectures have become widely adopted in recent LLMs~\citep{jiang2024mixtral,liu2024deepseek, xue2024openmoe,muennighoff2025olmoe}. By partitioning the network into multiple experts and routing each input to a sparse subset, MoE reduces per-token computation while improving scalability~\citep{shazeer2017outrageously, fedus2022switch}. Despite these advantages, MoE LLMs still suffer from a large parameter footprint due to expert replication~\citep{he2023pad}.
Moreover, unbalanced routing induces expert-level redundancy and highly skewed expert utilization, which creates substantial disparities in expert importance and complicates effective compression~\citep{lu2024not}.

\paragraph{MoE LLM Compression and Pruning}
SVD-based low-rank decomposition has been widely used as a structured compression tool for dense LLMs~\citep{hsu2022language,chen2021drone,yuan2025asvd,wang2025svdllm}.
For MoE LLMs, recent methods further exploit expert-level redundancy through pruning and structured decomposition.
MoE-I$^2$~\citep{yang2024moe} combines non-uniform inter-expert pruning with importance-aware intra-expert low-rank decomposition to compress MoE LLMs in a task-agnostic framework.
MoE-SVD~\citep{li2025moesvd} selectively decomposes less sensitive expert layers and reduces cross-expert redundancy through frequency-guided V-matrix sharing and U-matrix trimming.
D$^2$-MoE~\citep{gu2025delta} decomposes expert weights into a Fisher-weighted shared base and expert-specific delta weights, where the shared base is compressed via semi-dynamic pruning and the delta weights are compressed via truncation-aware SVD.

\paragraph{Post-Training Quantization for MoE LLMs}
Post-training quantization (PTQ) has become a widely used paradigm for compressing LLMs without retraining.
In this work, we focus on scalar weight quantization, a representative PTQ family that has been extensively studied for LLM compression. Among these methods, GPTQ~\citep{frantar-gptq} uses Hessian-based error compensation for sequential weight quantization, while HQQ~\citep{badri2023hqq} formulates low-bit quantization as a calibration-free half-quadratic optimization problem.
For MoE LLMs, MoEQuant~\citep{chen2025moequant} improves PTQ by constructing expert-balanced calibration samples and incorporating token expert affinities into the quantization process.
MiLo~\citep{huang2025milo} augments extremely quantized MoE models with adaptive low-rank compensators and efficient INT3 kernels to recover accuracy while improving inference efficiency.
MxMoE~\citep{duanmu2025mxmoe} assigns bit-widths according to block sensitivity, expert activation patterns, and hardware constraints, and generates optimized Group GEMM kernels for efficient MoE inference.
GEMQ~\citep{deng2026gemq} uses global linear programming to allocate expert bit-widths based on estimated task-loss increases and combines this allocation with router fine-tuning in a progressive quantization framework.

\section{Methodology}
\label{sec:method}

\subsection{Overview of \textsc{BitsMoE}}
\label{sec:bitsmoe}

\begin{figure*}[t]
    \centering
    \includegraphics[width=0.99\linewidth]{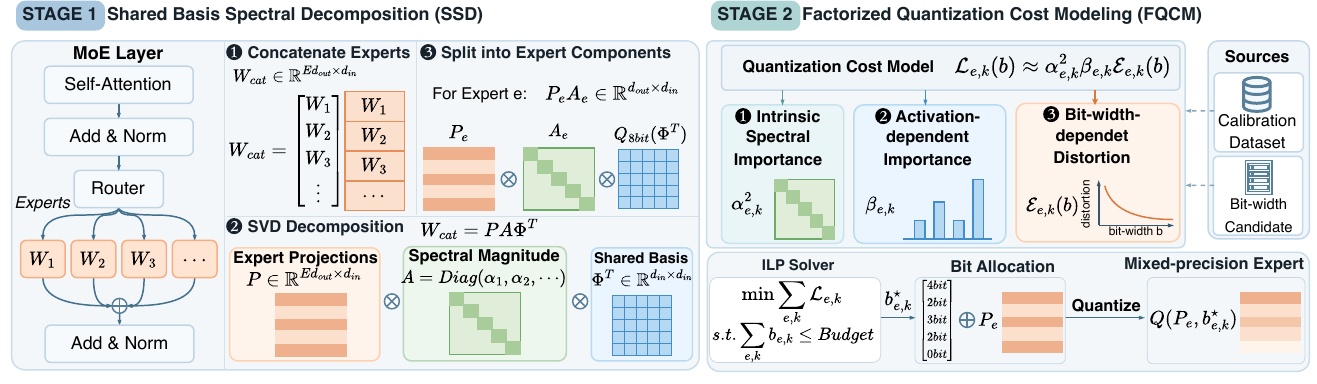}
    \caption{Overview of \textsc{BitsMoE}, a mixed-precision quantization framework for MoE LLMs.}
    \label{fig:overview}
\end{figure*}

We consider decoder-only MoE LLMs with $E$ routed experts in each MoE layer. For a token input $x$, the router selects a sparse subset $\mathcal{T}(x)\subseteq\{1,\ldots,E\}$. The routed-expert branch produces
\begin{equation}
    y=\sum_{e\in\mathcal{T}(x)}g_e(x)\,f_e(x),
    \label{eq:moe_routed_output}
\end{equation}
where $f_e$ denotes expert $e$ and $g_e(x)$ is its routing weight. Across tokens, these weights form the vector $g_e$ used below.

Fig.~\ref{fig:overview} presents \textsc{BitsMoE}, which addresses the two issues identified in Sec.~\ref{sec:intro} through the following components. We first construct a shared spectral basis across experts and identify expert-specific spectral directions as quantization units~(Sec.~\ref{sec:shared_basis_svd}). We then model their quantization costs to guide bit allocation under a fixed budget~(Sec.~\ref{sec:component_quant_cost}), and organize the quantized directions for packed GPU execution~(Sec.~\ref{sec:packed_execution}). Detailed derivations are provided in App.~\ref{app:piecewise_ilp_details}.

\subsection{Shared-Basis Spectral Decomposition (SSD)}
\label{sec:shared_basis_svd}

Within an MoE layer, expert weights operate on the same feature space and often exhibit substantial redundancy. Instead of decomposing each expert independently, we jointly decompose all experts to extract a shared spectral basis. We present a linear projection whose shared basis acts on its input space. App.~\ref{app:shared_basis_decomposition} provides the projection-specific formulations. Specifically, for expert weights $\{W_e\}_{e=1}^E$, where $W_e\in\mathbb R^{d_{\mathrm{out}}\times d_{\mathrm{in}}}$, we concatenate them and perform SVD:
\begin{equation}
    W_{\mathrm{cat}}
    :=
    \begin{bmatrix}
        W_1\\
        \vdots\\
        W_E
    \end{bmatrix}
    =
    U\boldsymbol{\Sigma}\boldsymbol{\Phi}^\top .
    \label{eq:ssd_joint}
\end{equation}
Here, $\boldsymbol{\Phi}=[\boldsymbol{\phi}_1,\ldots,\boldsymbol{\phi}_n]$ denotes the shared spectral basis acting on the common feature space across experts within an MoE layer, where $n=\min(E d_{\mathrm{out}},d_{\mathrm{in}})$ is the number of spectral components. Partitioning $U\boldsymbol{\Sigma}$ into expert-specific spectral components yields
\begin{equation}
    W_e
    =
    P_eA_e\boldsymbol{\Phi}^\top
    =
    \sum_{k=1}^n
    \boldsymbol{\alpha}_{e,k}p_{e,k}\boldsymbol{\phi}_k^\top ,
    \qquad e=1,\ldots,E,
    \label{eq:ssd_expert}
\end{equation}
where $p_{e,k}$ is the expert-specific spectral direction with $\|p_{e,k}\|_2=1$ and $\boldsymbol{\alpha}_{e,k}$ is the corresponding component magnitude. This representation decouples direction and magnitude within each spectral component $C_{e,k}=\boldsymbol{\alpha}_{e,k}p_{e,k}\boldsymbol{\phi}_k^\top$.


We define the spectral energy of each component as its squared Frobenius norm, which equals $\boldsymbol{\alpha}_{e,k}^2$:
\begin{equation}
    \lVert C_{e,k}\rVert_F^2
    =
    \operatorname{Tr}(C_{e,k}^\top C_{e,k})
    =
    \operatorname{Tr}\left[ (\boldsymbol{\alpha}_{e,k}\boldsymbol{\phi}_kp_{e,k}^\top) \boldsymbol{\alpha}_{e,k}p_{e,k}\boldsymbol{\phi}_k^\top\right]
    =
    \boldsymbol{\alpha}_{e,k}^2\operatorname{Tr}(\boldsymbol{\phi}_kp_{e,k}^\top p_{e,k}\boldsymbol{\phi}_k^\top)
    =
    \boldsymbol{\alpha}_{e,k}^2
    \label{eq:c2_norm}
\end{equation}

For an input activation matrix $X\in \mathbb{R}^{d_{\mathrm{in}}\times t}$ with $t$ tokens, the expert output can be expressed as
\begin{equation}
    Y_e=W_e X
    =
    \sum_{k=1}^n
    \boldsymbol{\alpha}_{e,k}p_{e,k}(\boldsymbol{\phi}_k^\top X)
    =
    \sum_{k=1}^n
    \boldsymbol{\alpha}_{e,k}p_{e,k}h_k
    \label{eq:quant_unit}
\end{equation}
Here, $h_k$ denotes the projection of the input activations onto $\boldsymbol{\phi}_k$. Since $\boldsymbol{\phi}_k$ is shared across experts, each projected input $h_k$ is associated with an expert-specific spectral direction $p_{e,k}$ and a corresponding magnitude $\boldsymbol{\alpha}_{e,k}$. We therefore use the expert-specific spectral direction $p_{e,k}$ of each spectral component as the quantization unit in \textsc{BitsMoE}. Unlike channel-wise or dimension-wise partitioning, this unit simultaneously captures the shared structure across experts and the intrinsic structure of each expert.

\subsection{Factorized Quantization Cost Modeling~(FQCM)}
\label{sec:component_quant_cost}

\begin{figure*}[t]
    \centering
    \begin{minipage}[t]{0.50\linewidth}
        \vspace{0pt}
        \centering
        \includegraphics[width=\linewidth]{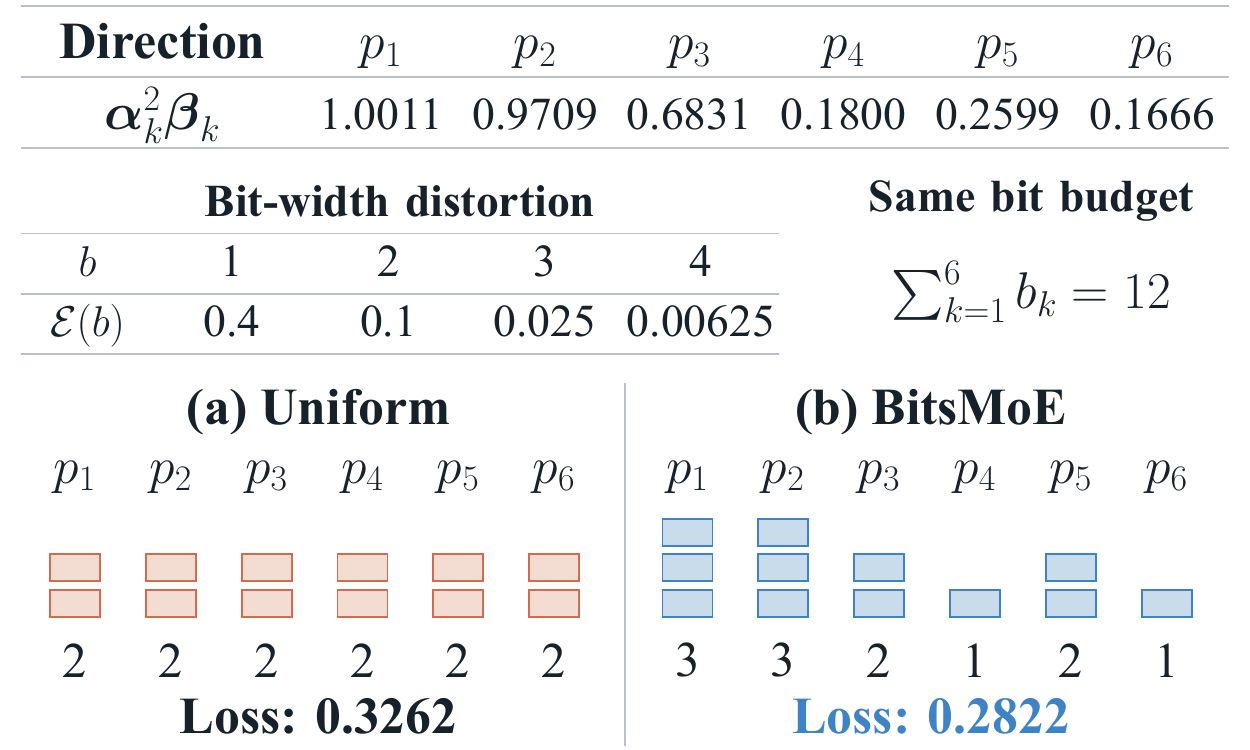}
        \caption{Uniform and \textsc{BitsMoE} bit allocations under the same budget.}
        \label{fig:bit_allocation_example}
    \end{minipage}\hfill
    \begin{minipage}[t]{0.475\linewidth}
        \vspace{0pt}
        \centering
        \includegraphics[width=\linewidth]{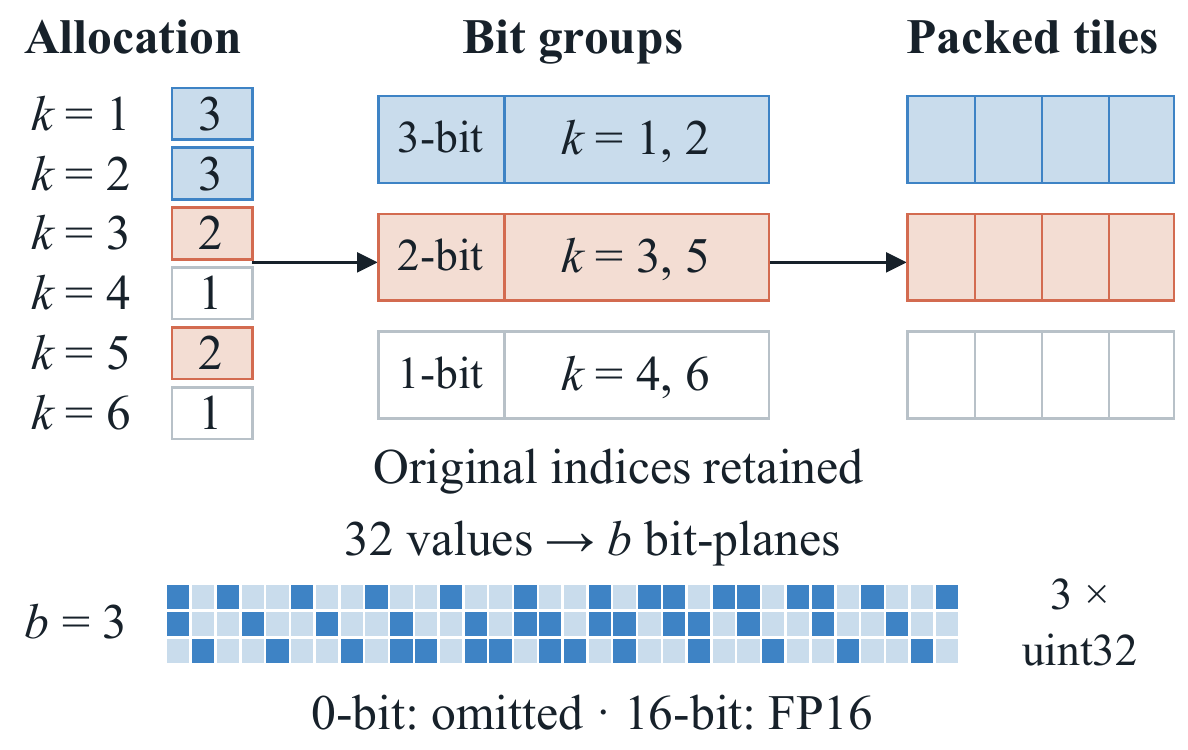}
        \caption{Components grouped by bit-width and packed into bit-planes.}
        \label{fig:packed_execution}
    \end{minipage}
\end{figure*}

Although SSD identifies spectral components as fine-grained quantization units within each MoE layer, the remaining question is how to allocate bits among them under a fixed budget. To guide this allocation, we model the cost of assigning each component a particular bit-width by jointly considering its intrinsic spectral importance, activation-dependent importance, and bit-width-dependent distortion. The illustrative example in Fig.~\ref{fig:bit_allocation_example} shows how this combined cost guides allocation: under the same total budget of 12 bits, cost-aware allocation reduces the modeled loss from 0.3262 to 0.2822 compared with uniform allocation.

We first quantize the shared basis to 8-bit precision using GPTQ~\citep{frantar-gptq} and denote the resulting basis by $\widehat{\boldsymbol{\Phi}}$. For the input-side form in Eq.~\eqref{eq:ssd_expert}, define $W_{e,q\phi}\coloneqq P_eA_e\widehat{\boldsymbol{\Phi}}^\top$ and $\widehat{W}_e(b)\coloneqq\widehat{P}_e(b)A_e\widehat{\boldsymbol{\Phi}}^\top$, with $\widehat{P}_e(b)\coloneqq\bigl[Q_b(p_{e,1}),\ldots,Q_b(p_{e,n})\bigr]$. We measure the additional error introduced by quantizing the expert-specific spectral directions relative to $W_{e,q\phi}$.
For per-token routing weights $g_e$, let $G_e\coloneqq\operatorname{Diag}(g_e)$ and $H_{g,e}\coloneqq XG_eX^\top$. Let $\boldsymbol{\varepsilon}_{e,k,b}=p_{e,k}-Q_b(p_{e,k})$ denote the error in each spectral direction. The weight error expands as
\begin{equation}
    W_{e,q\phi}-\widehat{W}_e(b)
    =\bigl(P_e-\widehat{P}_e(b)\bigr)A_e\widehat{\boldsymbol{\Phi}}^\top
    =\sum_{k=1}^{n}\boldsymbol{\alpha}_{e,k}\boldsymbol{\varepsilon}_{e,k,b}\widehat{\boldsymbol{\phi}}_k^\top.
    \label{eq:direction_weight_error}
\end{equation}
Under this metric, we define the expected routing-weighted output reconstruction loss from quantizing the spectral directions as
\begin{align}
        \mathcal{L}_{e,P}(b)
        &\coloneqq
        \mathbb{E}_{\epsilon}\left\lVert
            \bigl(W_{e,q\phi}-\widehat{W}_e(b)\bigr)XG_e^{1/2}
        \right\rVert_F^2
        =\mathbb{E}_{\epsilon}\operatorname{Tr}\!\left[
            \bigl(W_{e,q\phi}-\widehat{W}_e(b)\bigr)
            H_{g,e}
            \bigl(W_{e,q\phi}-\widehat{W}_e(b)\bigr)^\top
        \right]\notag\\
        &=\sum_{k,l}\boldsymbol{\alpha}_{e,k}\boldsymbol{\alpha}_{e,l}
        \bigl(\widehat{\boldsymbol{\phi}}_k^\top H_{g,e}\widehat{\boldsymbol{\phi}}_l\bigr)
        \mathbb{E}_{\epsilon}\!\left[
            \boldsymbol{\varepsilon}_{e,l,b}^\top\boldsymbol{\varepsilon}_{e,k,b}
        \right].
    \label{eq:activation_output_objective_process}
\end{align}
$\mathbb{E}_{\epsilon}$ denotes expectation over the quantization errors. Under the diagonal component-error approximation~(App.~\ref{app:activation_weighted_loss}), we approximate $\mathbb{E}_{\epsilon}[\boldsymbol{\varepsilon}_{e,l,b}^\top\boldsymbol{\varepsilon}_{e,k,b}]$ by zero for $k\ne l$. Define $\boldsymbol{\beta}_{e,k}\coloneqq\widehat{\boldsymbol{\phi}}_k^\top H_{g,e}\widehat{\boldsymbol{\phi}}_k$ and $\mathcal{E}_{e,k,b}\coloneqq\mathbb{E}_{\epsilon}\lVert\boldsymbol{\varepsilon}_{e,k,b}\rVert_2^2$. The remaining diagonal terms give
\begin{equation}
    \mathcal{L}_{e,P}(b)
    \approx
    \sum_k
    \boldsymbol{\alpha}_{e,k}^2\boldsymbol{\beta}_{e,k}\mathcal{E}_{e,k,b}.
    \label{eq:activation_output_objective_formula}
\end{equation}

This formulation decomposes the component-wise quantization cost of each spectral direction \(p_{e,k}\) at bit-width \(b\) into three interpretable factors:
\textbf{(1)~Intrinsic spectral importance.}
We define $\boldsymbol{\alpha}_{e,k}^2$ as the intrinsic spectral importance, which measures the calibration-agnostic strength of a spectral component in the expert weight transformation.
\textbf{(2)~Activation-dependent importance.}
We define $\boldsymbol{\beta}_{e,k}$ as the activation-dependent importance, which captures how the calibration inputs and their activation patterns determine the importance of the spectral component.
For bit allocation, we reuse $\boldsymbol{\beta}_{e,k}$ for its normalized diagonal approximation, as described in App.~\ref{app:activation_weighted_loss}.
\textbf{(3)~Bit-width-dependent distortion.}
We define $\mathcal{E}_{e,k,b}$ as the bit-width-dependent distortion, which measures the quantization error incurred when assigning bit-width \(b\) to the spectral component.
Under the high-resolution approximation for symmetric quantization, distortion decays proportionally to $\exp(-\lambda b)$ under 6-bit and 8-bit quantization, where $\lambda=2\ln 2$. At lower bit-widths, we use shared empirical coefficients for 2-bit, 3-bit, and 4-bit quantization and a direction-dependent sign-quantization error for 1-bit quantization, as detailed in App.~\ref{app:piecewise_reconstruction_error}.

We model the component-wise quantization cost of assigning bit-width $b$ to $p_{e,k}$ as
\begin{equation}
    \label{eq:component_cost}
    \mathcal{L}_{e,k}(b)
    \approx
    \underbrace{
        \boldsymbol{\alpha}_{e,k}^{2}
    }_{\substack{\text{Intrinsic spectral}\\\text{importance}}}
    \;\cdot\;
    \underbrace{
        \boldsymbol{\beta}_{e,k}
    }_{\substack{\text{Activation-dependent}\\\text{importance}}}
    \;\cdot\;
    \underbrace{
        \mathcal{E}_{e,k,b}
    }_{\substack{\text{Bit-width-dependent}\\\text{distortion}}}.
\end{equation}

\subsection{Cost-Aware Bit Allocation and GPU Execution}
\label{sec:ilp_model}
\label{sec:packed_execution}
Let $\ell\in[L]$ index the $L$ MoE layers and $y_{\ell,e,k,b}\in\{0,1\}$ indicate whether component $(\ell,e,k)$ is assigned bit-width $b\in\mathcal{B}=\{16,8,6,4,3,2,1,0\}$. Then we define the FQCM loss as:
\begin{equation}
    \label{eq:fqcm_loss}
    \mathcal{L}_{\mathrm{FQCM}}(y)
    \coloneqq
    \sum_{\ell=1}^{L}\sum_{e,k}\sum_{b\in\mathcal{B}}
    \mathcal{L}_{\ell,e,k}(b)y_{\ell,e,k,b},
\end{equation}
We minimize this global FQCM loss through a single ILP problem:
\begin{equation}
    \label{eq:component_bit_allocation}
    \begin{aligned}
        \textsc{Minimize}\quad
        &\mathcal{L}_{\mathrm{FQCM}}(y)\\
        \textsc{s.t.}\quad
        &\sum_{b\in\mathcal{B}}y_{\ell,e,k,b}=1,
        \quad \forall\ell,e,k,\\
        &\sum_{\ell,e,k}\sum_{b\in\mathcal{B}}
        b\,y_{\ell,e,k,b}\leq B_{\mathrm{alloc}}(\mathfrak{b}_{\mathrm{eq}}).
    \end{aligned}
\end{equation}
Here, $B_{\mathrm{alloc}}(\mathfrak{b}_{\mathrm{eq}})$ denotes the global normalized bit budget for a target bit-width $\mathfrak{b}_{\mathrm{eq}}$. We solve this global ILP independently for each projection type in $\{\mathrm{gate\_proj},\mathrm{up\_proj},\mathrm{down\_proj}\}$. The detailed derivation is provided in App.~\ref{app:equivalent_bit_budget}.

\textbf{Packed storage and GPU execution.}
We group spectral components by bit-width within each expert and preserve their original indices for indexed computation. Fig.~\ref{fig:packed_execution} illustrates this organization. Each low-bit group uses a bit-plane layout: 32 $b$-bit values occupy $b$ \texttt{uint32} words, one per bit position. Zero-bit directions are omitted, and 16-bit directions remain in FP16. During inference, kernels unpack and dequantize the stored directions on the fly to compute expert contributions without reconstructing dense expert weights. Within each MoE layer, we apply each shared-basis transform once per token across all routed experts. We use vLLM's Marlin W8A16 kernels~\citep{frantar2025marlin} for these shared-basis matrix multiplications.

\section{Experiments}
\label{sec:experiments}

\begin{table*}[t]
    \caption{Task-wise accuracy (\%) of different methods on the seven downstream tasks.}
    \label{tab:dsv2_qw3_qw3next_results}
    \centering
    \tablestyle
    \begin{tabular}{@{}ccccccccccc@{}}
        \toprule
        \multirow{2}{*}{\textbf{Model}} & \multirow{2}{*}{\textbf{Method}} & \multirow{2}{*}{\textbf{Bits}} & \multicolumn{8}{c}{\textbf{Accuracy} (\%) $\uparrow$} \\
        \cmidrule(lr){4-11}
        & & & \textbf{HS} & \textbf{AE} & \textbf{MQ} & \textbf{MM} & \textbf{OB} & \textbf{WG} & \textbf{GS} & \textbf{Avg.} \\
        \midrule

        \multirow{9}{*}{\makecell{\textsc{DeepSeek-}\\\textsc{V2-Lite}}}
         & FP16 & 16 & 77.70 & 76.30 & 39.03 & 55.60 & 44.40 & 70.88 & 39.12 & 57.58 \\
        \cmidrule(lr){2-11}
        & HQQ & 3 & \textbf{76.83} & 71.72 & 36.45 & \textbf{53.16} & \textbf{44.40} & \textbf{70.88} & \textbf{32.15} & \textbf{55.08} \\
        & GPTQ & 3 & 75.88 & 73.48 & 37.29 & 50.92 & 44.20 & 69.30 & 30.40 & 54.50 \\
        & \textbf{\textsc{BitsMoE}} & 3 & 74.65 & \textbf{74.24} & \textbf{38.06} & 52.86 & 41.60 & 70.01 & 29.34 & 54.39 \\
        \cmidrule(lr){2-11}
         & HQQ & 2 & 67.79 & 61.45 & 29.25 & \textbf{43.48} & 38.00 & 63.85 & 11.98 & 45.11 \\
         & GPTQ & 2 & \textbf{68.81} & 65.74 & 28.21 & 37.66 & 36.40 & 65.75 & 10.08 & 44.66 \\
        \cmidrule(lr){2-11}
         & \multirow[c]{3}{*}{\textbf{\textsc{BitsMoE}}} & 2.0 & 68.21 & \textbf{70.54} & \textbf{32.70} & 41.78 & \textbf{39.60} & \textbf{67.01} & \textbf{12.21} & \textbf{47.44} \\
         &  & 1.9 & 66.33 & 69.15 & 31.76 & 42.00 & 38.80 & 67.32 & 11.14 & 46.64 \\
         &  & 1.8 & 64.19 & 66.54 & 30.75 & 40.09 & 37.80 & 66.06 & 9.10 & 44.93 \\
        \midrule

        \multirow{11}{*}{\makecell{\textsc{Qwen3-}\\\textsc{30B-A3B}}}
         & FP16 & 16 & 81.35 & 79.80 & 60.03 & 78.77 & 45.00 & 72.85 & 83.47 & 71.61 \\
        \cmidrule(lr){2-11}
        & HQQ & 3 & 78.55 & 76.56 & 49.01 & 75.37 & 44.40 & 71.67 & 79.53 & 67.87 \\
        & GPTQ & 3 & \textbf{79.92} & 78.11 & 54.07 & \textbf{75.83} & 43.40 & 72.14 & 79.45 & 68.99 \\
        & \textbf{\textsc{BitsMoE}} & 3 & 78.58 & \textbf{79.88} & \textbf{56.21} & 75.64 & \textbf{46.80} & \textbf{74.27} & \textbf{82.64} & \textbf{70.57} \\
        \cmidrule(lr){2-11}
         & HQQ & 2 & 63.05 & 37.33 & 33.17 & 48.82 & 36.40 & 60.85 & 22.67 & 43.18 \\
         & GPTQ & 2 & 70.16 & 59.97 & 24.89 & 39.17 & 39.40 & 60.62 & 4.32 & 42.65 \\
        \cmidrule(lr){2-11}
         & \multirow[c]{5}{*}{\textbf{\textsc{BitsMoE}}} & 2.0 & \textbf{70.43} & \textbf{80.18} & \textbf{52.86} & \textbf{64.36} & \textbf{41.00} & \textbf{70.17} & \textbf{71.04} & \textbf{64.29} \\
         &  & 1.9 & 68.57 & 77.99 & 51.22 & 61.78 & 38.60 & 69.53 & 66.87 & 62.08 \\
         &  & 1.8 & 66.21 & 76.14 & 48.41 & 60.40 & 37.00 & 67.40 & 64.75 & 60.04 \\
         &  & 1.7 & 63.00 & 74.03 & 46.43 & 54.94 & 36.20 & 65.43 & 58.15 & 56.88 \\
         &  & 1.6 & 60.57 & 73.70 & 43.45 & 48.18 & 35.00 & 65.04 & 54.06 & 54.29 \\
        \midrule

        \multirow{11}{*}{\makecell{\textsc{Qwen3-Next-}\\\textsc{80B-A3B}}}
         & FP16 & 16 & 82.72 & 84.64 & 63.85 & 84.53 & 44.20 & 76.80 & 77.18 & 73.42 \\
        \cmidrule(lr){2-11}
        & HQQ & 3 & \textbf{82.11} & 83.50 & 61.17 & \textbf{83.81} & \textbf{45.60} & \textbf{77.03} & 76.57 & 72.83 \\
        & GPTQ & 3 & 81.42 & 84.30 & 59.40 & 82.52 & 44.20 & 76.09 & 76.65 & 72.08 \\
        & \textbf{\textsc{BitsMoE}} & 3 & 81.23 & \textbf{85.35} & \textbf{63.25} & 83.79 & 44.20 & 76.16 & \textbf{77.03} & \textbf{73.00} \\
        \cmidrule(lr){2-11}
         & HQQ & 2 & \textbf{78.73} & 83.50 & 50.95 & 79.68 & 43.80 & 70.40 & 66.49 & 67.65 \\
         & GPTQ & 2 & 70.24 & \textbf{84.30} & 27.47 & 54.63 & 38.60 & 65.59 & 18.35 & 51.31 \\
        \cmidrule(lr){2-11}
         & \multirow[c]{5}{*}{\textbf{\textsc{BitsMoE}}} & 2.0 & 78.05 & 83.92 & \textbf{61.71} & \textbf{79.95} & \textbf{44.60} & \textbf{74.90} & \textbf{71.80} & \textbf{70.70} \\
         &  & 1.9 & 77.45 & 83.42 & 61.74 & 79.11 & 44.40 & 75.22 & 70.28 & 70.23 \\
         &  & 1.8 & 76.37 & 83.04 & 59.80 & 77.93 & 43.80 & 74.59 & 69.07 & 69.23 \\
         &  & 1.7 & 75.07 & 84.09 & 58.46 & 77.13 & 44.00 & 74.74 & 69.60 & 69.01 \\
         &  & 1.6 & 73.39 & 83.63 & 56.52 & 75.97 & 42.60 & 74.19 & 66.79 & 67.58 \\
        \bottomrule
    \end{tabular}
\end{table*}

\subsection{Experimental Setup}
\paragraph{Models and Datasets.}
We evaluate three MoE LLMs: DeepSeek-V2-Lite~\citep{liu2024deepseek}, Qwen3-30B-A3B-Base~\citep{qwen3} and Qwen3-Next-80B-A3B-Instruct~\citep{qwen3next}. Hereafter, the two Qwen checkpoints are abbreviated as Qwen3-30B-A3B and Qwen3-Next-80B-A3B.
We evaluate the proposed \textsc{BitsMoE} on six zero-shot tasks, including ARC-Easy~\citep{clark2018think}, HellaSwag~\citep{zellers2019hellaswag}, MathQA~\citep{amini2019mathqa}, Massive Multitask Language Understanding (MMLU)~\citep{hendrycks2009measuring}, OpenBookQA~\citep{mihaylov2018can} and WinoGrande~\citep{sakaguchi2021winogrande}. Furthermore, we evaluate \textsc{BitsMoE} using Grade School Math 8K (GSM8K)~\citep{cobbe2021training}, which assesses multi-step mathematical reasoning skills.
We evaluate these seven tasks on NVIDIA A100 80GB PCIe GPUs using the open-source tool lm-evaluation-harness (version 0.4.9.1)~\citep{eval-harness}, and the ILP problems are solved using the Gurobi Optimizer~\citep{gurobi}.

\paragraph{Baselines.}

Our baselines include representative LLM post-training quantization (PTQ) methods HQQ~\citep{badri2023hqq} and GPTQ~\citep{frantar-gptq}. On Qwen3-30B-A3B, we select MoEQuant~\citep{chen2025moequant}, MiLo~\citep{huang2025milo}, MxMoE~\citep{duanmu2025mxmoe}, and GEMQ~\citep{deng2026gemq} as representative MoE-specific PTQ methods covering GPTQ-based quantization, low-rank decomposition for quantization-error compensation, hardware-aware ILP-based mixed-precision quantization, and global expert-level mixed-precision quantization, respectively. All methods quantize only MoE expert weights with group size 128.

\subsection{Main Results}
\label{tab:main_results}

Tab.~\ref{tab:dsv2_qw3_qw3next_results} shows that \textsc{BitsMoE} achieves the highest average accuracy at 2-bit on all three models while remaining competitive at 3-bit. On Qwen3-Next-80B-A3B, its 2-bit average accuracy of 70.70\% exceeds HQQ by 3.05 percentage points. The advantage extends to mathematical reasoning below 2-bit precision. At 1.6-bit, Qwen3-30B-A3B retains 54.06\% accuracy on GSM8K and 43.45\% on MathQA, which exceed those of 2-bit HQQ by 31.39 and 10.28 percentage points, respectively. Mathematical reasoning is crucial for many real-world applications, so preserving this ability at ultra-low precision highlights the practical relevance of \textsc{BitsMoE}. Tab.~\ref{tab:wikitext2_ppl} reports WikiText-2~\citep{merity2016pointer} perplexity on the Qwen3-30B-A3B base model. \textsc{BitsMoE} achieves the lowest 2-bit perplexity of 9.12, which is 0.20 lower than GPTQ's 9.32.

As shown in Tab.~\ref{tab:comparison_moequant_milo_full}, \textsc{BitsMoE} achieves the highest average accuracy of 64.29\% at 2-bit on Qwen3-30B-A3B among MoE-specific methods, including those based on ILP allocation and low-rank error compensation. In particular, \textsc{BitsMoE} achieves 71.04\% accuracy on GSM8K and surpasses GEMQ by 24.94 percentage points. This advantage further supports the practical value of \textsc{BitsMoE} for real-world applications that require multi-step mathematical reasoning.

\begin{table}[h]
    \centering
    \begin{minipage}[t]{0.28\linewidth}
        \vspace{0pt}
        \centering
        \caption{PPL on Qwen3-30B-A3B.}
        \label{tab:wikitext2_ppl}
        \tablestyle
        \setlength{\tabcolsep}{4pt}
        \begin{tabular}{@{}ccc@{}}
            \toprule
            \textbf{Method} & \textbf{Bits} & \textbf{W2 PPL}$\downarrow$ \\
            \midrule
            FP16 & 16 & 5.51 \\
            \midrule
            HQQ & \multirow{5}{*}{2} & 12.52 \\
            GPTQ & & 9.32 \\
            MiLo & & 11.74 \\
            MoEQuant & & 12.04 \\
            \textbf{\textsc{BitsMoE}} & & \textbf{9.12} \\
            \bottomrule
        \end{tabular}
    \end{minipage}\hfill
    \begin{minipage}[t]{0.70\linewidth}
        \vspace{0pt}
        \centering
        \caption{Evaluation results on Qwen3-30B-A3B under 2-bit quantization using MoE-specific methods.}
        \label{tab:comparison_moequant_milo_full}
        \tablestyle
        \setlength{\tabcolsep}{4pt}
        \begin{tabular}{@{}c*{8}{c}@{}}
            \toprule
            \multirow{2}{*}{\textbf{Method}} & \multicolumn{8}{c}{\textbf{Accuracy} (\%) $\uparrow$} \\
            \cmidrule(lr){2-9}
            & \textbf{HS} & \textbf{AE} & \textbf{MQ} & \textbf{MM} & \textbf{OB} & \textbf{WG} & \textbf{GS} & \textbf{Avg.} \\
            \midrule
            MiLo & 62.82 & 58.71 & 31.83 & 44.51 & 35.40 & 59.98 & 19.64 & 44.70 \\
            MxMoE & \underline{76.91} & 59.30 & 22.81 & 48.45 & 40.40 & \underline{71.27} & 1.90 & 45.86 \\
            MoEQuant & 66.44 & 69.70 & \underline{45.09} & \textbf{70.02} & 40.00 & 67.72 & \underline{49.36} & 58.33 \\
            GEMQ & \textbf{78.68} & \textbf{81.44} & 41.61 & 64.26 & \textbf{44.00} & \textbf{74.35} & 46.10 & \underline{61.49} \\
            \textbf{\textsc{BitsMoE}} & 70.43 & \underline{80.18} & \textbf{52.86} & \underline{64.36} & \underline{41.00} & 70.17 & \textbf{71.04} & \textbf{64.29} \\
            \bottomrule
        \end{tabular}
    \end{minipage}
\end{table}

\begin{table}[h]
    \centering
    \begin{minipage}[t]{0.51\linewidth}
        \vspace{0pt}
        \centering
        \caption{Average accuracy for ablations on Qwen3-30B-A3B at 2-bit. The Bits column specifies shared-basis storage precision.}
        \label{tab:ablation_combined}
        \tablestyle
        \setlength{\tabcolsep}{4pt}
        \begin{tabular}{@{}ccccccc@{}}
            \toprule
            \textbf{Basis} & \textbf{Bits} & \textbf{Alloc.} & $\boldsymbol{\alpha}^2$ & $\boldsymbol{\beta}$ & $G_e$ & \makecell{\textbf{Avg. Acc.}\\(\%) $\uparrow$} \\
            \midrule
            Non-Basis & -- & Uniform & -- & -- & -- & 24.86 \\
            \midrule
            \multirow{9}{*}[-5pt]{\textsc{SSD}} & \multirow{6}{*}[-2.5pt]{8-bit} & Uniform & -- & -- & -- & 25.27 \\
            \cmidrule(l){3-7}
             &  & \multirow{5}{*}{FQCM} & \xmark & \xmark & -- & 26.17 \\
             &  &  & \cmark & \xmark & -- & 26.48 \\
             &  &  & \xmark & \cmark & \xmark & 47.20 \\
             &  &  & \xmark & \cmark & \cmark & 57.92 \\
             &  &  & \cmark & \cmark & \xmark & 60.45 \\
            \cmidrule(l){2-7}
             & FP16 & \multirow{3}{*}{FQCM} & \cmark & \cmark & \cmark & 60.91 \\
             & 4-bit &  & \cmark & \cmark & \cmark & 62.78 \\
             & 8-bit &  & \cmark & \cmark & \cmark & \textbf{64.29} \\
            \bottomrule
        \end{tabular}
    \end{minipage}\hfill
    \begin{minipage}[t]{0.47\linewidth}
        \vspace{0pt}
        \centering
        \includegraphics[width=\linewidth,trim=0 0 0 16bp,clip]{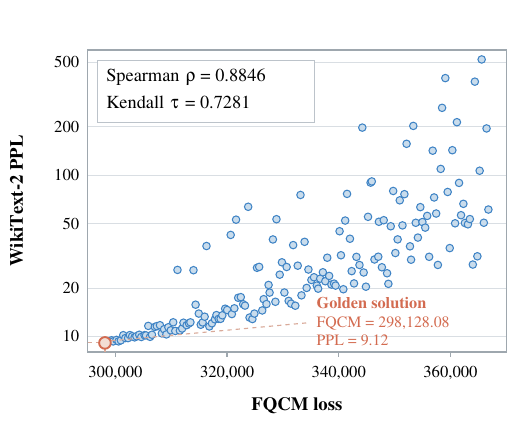}
        \captionof{figure}{Correlation between FQCM loss and WikiText-2 PPL for 170 allocations at 2-bit on Qwen3-30B-A3B.}
        \label{fig:fqcm_ppl_fixed2bit}
    \end{minipage}
\end{table}

\subsection{Ablation Study}
\label{sec:ablation}

\paragraph{Component importance and shared-basis precision.}
Tab.~\ref{tab:ablation_combined} evaluates SSD, FQCM, and shared-basis precision on Qwen3-30B-A3B at 2-bit. Non-Basis applies SVD independently to each expert, whereas SSD constructs a shared basis. Uniform assigns equal budgets to all experts and quantizes spectral components to 2-bit in descending spectral energy order until each expert's budget is exhausted, assigning the remaining components to 0-bit. FQCM adaptively allocates bits using the cost model. Uniform allocation achieves only 24.86\% average accuracy with independent SVD and 25.27\% with SSD, suggesting that basis sharing alone is insufficient under this allocation scheme. With an 8-bit shared basis, using spectral energy $\boldsymbol{\alpha}^2$ as the sole importance signal achieves 26.48\%: intrinsic spectral importance alone appears insufficient to guide component-wise bit allocation under actual activation patterns. Using only activation-dependent importance $\boldsymbol{\beta}$ instead accounts for observed activation patterns, raising accuracy to 47.20\% without routing weights. Incorporating $G_e$ accounts for expert-specific routing weights and further improves accuracy to 57.92\%. Adding $\boldsymbol{\alpha}^2$ then reaches 64.29\%, suggesting that spectral energy within experts complements their activation and routing information. Within the FQCM, 8-bit also outperforms FP16 and 4-bit shared bases. FP16 consumes more of the limited budget, leaving fewer bits for expert-specific components, whereas 4-bit errors may degrade the common transformations used by all experts within a MoE layer. We therefore adopt an 8-bit shared basis to balance storage cost and accuracy while supporting efficient computation with W8A16 kernels. The corresponding task-level results are reported in Tab.~\ref{tab:ablation_combined_full} in App.~\ref{app:ablation_supplement}.

\paragraph{FQCM as an allocation criterion.}
In Fig.~\ref{fig:fqcm_ppl_fixed2bit}, we randomly perturb the ILP bit allocation while preserving the 2-bit budget and evaluate WikiText-2 PPL for 170 allocations, including the ILP solution. Spearman's $\rho=0.8846$ and Kendall's $\tau=0.7281$ indicate a strong positive rank correlation between FQCM loss and PPL. PPL becomes more dispersed at higher FQCM losses, possibly reflecting nonlinear error propagation and limiting precise PPL prediction. In the low-loss region, FQCM loss and PPL exhibit a particularly strong positive rank correlation, supporting FQCM as an objective for ILP bit allocation. The ILP solution obtained by minimizing FQCM loss under the fixed budget lies in this region and achieves the lowest measured PPL among the evaluated allocations. This experiment validates the effectiveness of our FQCM for bit allocation.

\subsection{Efficiency Analysis}
\label{sec:efficiency_analysis}

\begin{figure}[h]
    \centering
    \begin{subfigure}[t]{0.49\linewidth}
        \centering
        \includegraphics[width=\linewidth]{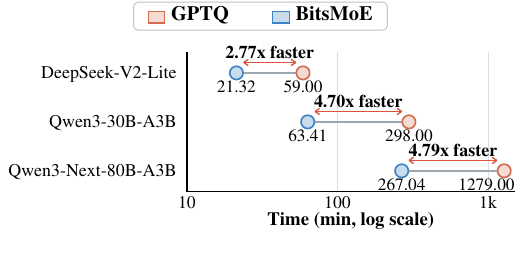}
        \caption{Quantization time $\downarrow$.}
        \label{fig:quantization_time}
    \end{subfigure}\hfill
    \begin{subfigure}[t]{0.49\linewidth}
        \centering
        \includegraphics[width=\linewidth]{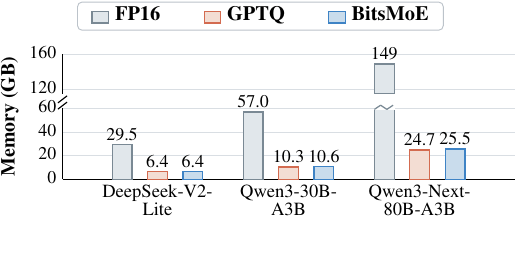}
        \caption{Memory footprint $\downarrow$.}
        \label{fig:memory_footprint}
    \end{subfigure}
    \caption{Quantization efficiency and memory footprint at 2-bit.}
    \label{fig:efficiency_summary}
\end{figure}

\begin{figure}[h]
    \centering
    \includegraphics[width=\textwidth]{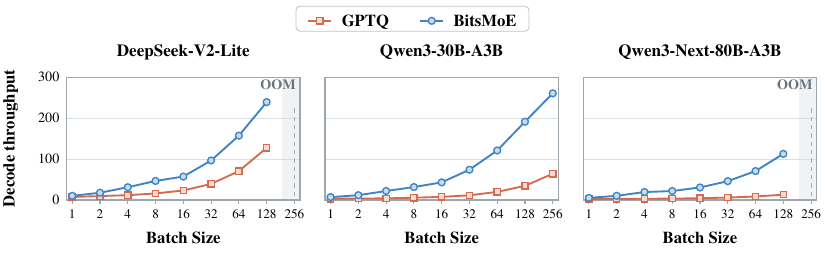}
    \caption{Decode throughput (tok/s) of GPTQ and \textsc{BitsMoE} at 2-bit on DeepSeek-V2-Lite, Qwen3-30B-A3B and Qwen3-Next-80B-A3B.}
    \label{fig:decode_throughput}
\end{figure}

\paragraph{Offline Quantization Overhead.}
On Qwen3-Next-80B-A3B, \textsc{BitsMoE} reduces end-to-end offline quantization time on an NVIDIA A100 PCIe (80\,GB) GPU from GPTQ's 1,279 minutes to 267.04 minutes, a $4.79\times$ speedup (Fig.~\ref{fig:quantization_time}). On Qwen3-30B-A3B at 2-bit, it achieves $49.00\times$ and $16.47\times$ end-to-end offline quantization speedups over the ILP-based methods MxMoE and GEMQ, respectively. FQCM computes component costs directly for bit allocation, avoiding repeated empirical evaluation of candidate output errors. App.~\ref{app:quantization_time_breakdown} reports per-stage timings and the ILP-based comparison (Tab.~\ref{tab:ilp_quantization_time}).

\paragraph{Inference Efficiency.}
\textsc{BitsMoE} also reduces the GPU memory footprint of Qwen3-Next-80B-A3B from 148.69\,GB in FP16 to 25.46\,GB at 2-bit, a $5.84\times$ reduction (Fig.~\ref{fig:memory_footprint}). For decoding, we use the Triton backend~\citep{tillet2019triton} for GPTQ because the Marlin~\citep{frantar2025marlin} and ExLlamaV2~\citep{exllamav2} backends in our GPTQModel implementation do not support 2-bit GPTQ checkpoints. On a single NVIDIA A6000 GPU, \textsc{BitsMoE} achieves higher decode throughput than GPTQ on all three models across all measured batch sizes (Fig.~\ref{fig:decode_throughput}). Decode speedups over GPTQ reach up to $2.90\times$, $6.46\times$, and $8.71\times$ on DeepSeek-V2-Lite, Qwen3-30B-A3B, and Qwen3-Next-80B-A3B, respectively. The absolute throughput gap widens as batch size increases. App.~\ref{app:absolute_efficiency} provides further results and analysis on inference efficiency.

\section{Conclusion}

We presented \textsc{BitsMoE}, a fine-grained MoE quantization framework that defines structurally meaningful quantization units and models their costs under actual activation patterns and different bit-widths. SSD separates shared structure across experts from expert-specific spectral components, which serve as quantization units. FQCM combines intrinsic spectral importance, activation-dependent importance, and bit-width-dependent distortion to estimate component-wise costs for global ILP-based bit allocation under a fixed memory budget. Experiments on three MoE models and seven downstream tasks demonstrate improved average accuracy at 2-bit. On Qwen3-30B-A3B at 2-bit, \textsc{BitsMoE} achieves 64.29\% average accuracy, exceeding GEMQ by 2.80 percentage points. Under the same setting, it achieves a $16.47\times$ end-to-end offline quantization speedup over GEMQ. It also achieves up to $6.46\times$ the decode throughput of GPTQ. At 1.4-bit, it achieves 45.07\% average accuracy, surpassing 2-bit GPTQ's 42.65\% by 2.42 percentage points.


\subsubsection*{Acknowledgments}
This work was partially supported by the Strategic Priority Research Program of the CAS under Grant XDB0660000, and in part by the National Natural Science Foundation of China under Grant 92473114.

This work was partially supported by the Ministry of Education, Singapore, under the Academic Research Fund Tier 2 (MOE-T2EP20224-0006).

\bibliography{iclr27}
\bibliographystyle{iclr27}

\clearpage
\appendix
\startcontents[appendix]
\appendixorganization
\clearpage
\section*{Notation}
\phantomsection
\addcontentsline{toc}{section}{Notation}
\label{app:notation}

The notation table (Tab.~\ref{tab:app_detailed_notation}) collects the core quantities used in App.~\ref{app:piecewise_ilp_details}. We use $\ell\in[L]$ for MoE layers, $e\in[E]$ for experts, and $k\in[n^{(h)}]$ for spectral components.

\begin{table}[h!]
    \centering
    \tablestyle
    \caption{Core notation for App.~\ref{app:piecewise_ilp_details}.}
    \label{tab:app_detailed_notation}
    \renewcommand{\tabularxcolumn}[1]{m{#1}}
    \begin{tabularx}{\linewidth}{@{}>{\centering\arraybackslash}m{0.33\linewidth}X@{}}
        \toprule
        \textbf{Symbols} & \textbf{Description} \\
        \midrule
        \multicolumn{2}{c}{\textbf{Projection Types and Dimensions}} \\
        \midrule
        $h\in\mathcal{H}$, $\mathcal{H}_{\mathrm{in}}$, $h_{\mathrm{dn}}$
        & Projection type in $\mathcal{H}=\{\mathtt{gate\_proj},\mathtt{up\_proj},\mathtt{down\_proj}\}$; $\mathcal{H}_{\mathrm{in}}$ contains gate/up, and $h_{\mathrm{dn}}$ denotes down. \\
        \midrule
        $m$, $n$, $n^{(h)}$
        & Spectral-direction dimension $m=d_{\mathrm{out}}$ and spectral-component count $n=n^{(h)}=d_{\mathrm{in}}$ in the full-component setting; $d_{\mathrm{in}},d_{\mathrm{out}}$ are $\mathtt{gate\_proj}$/$\mathtt{up\_proj}$ input/output dimensions. \\
        \midrule
        \multicolumn{2}{c}{\textbf{Shared-basis Spectral Decomposition}} \\
        \midrule
        $\boldsymbol{\Phi}^{(h)}$, $\boldsymbol{\phi}_k^{(h)}$
        & Shared spectral basis across experts within a layer, and its $k$-th column. \\
        \midrule
        $P_e^{(h)}$, $p_{e,k}^{(h)}$
        & Normalized expert-specific spectral matrix and its $k$-th column, a unit-norm expert-specific spectral direction. \\
        \midrule
        $A_e^{(h)}$, $\boldsymbol{\alpha}_{e,k}^{(h)}$
        & Diagonal magnitude matrix and component magnitude. \\
        \midrule
        $C_{e,k}^{(h)}$, $\bigl(\boldsymbol{\alpha}_{e,k}^{(h)}\bigr)^2$
        & Rank-one spectral component and its spectral energy (intrinsic spectral importance). \\
        \midrule
        \multicolumn{2}{c}{\textbf{Factorized Quantization Cost Modeling}} \\
        \midrule
        $\widehat{\boldsymbol{\Phi}}^{(h)}$, $\widehat{\boldsymbol{\phi}}_k^{(h)}$
        & Shared spectral basis quantized to 8-bit precision using GPTQ, and its $k$-th column. \\
        \midrule
        $W_{e,q\phi}^{(h)}$
        & Reference weight with only the shared spectral basis quantized. \\
        \midrule
        $H_{g,e}^{(h)}$
        & Routing-weighted Hessian. \\
        \midrule
        $\boldsymbol{\beta}_{e,k}^{(h)}$
        & Activation-dependent importance. \\
        \midrule
        $\boldsymbol{\varepsilon}_{e,k,b}^{(h)}$
        & Quantization error vector of the expert-specific spectral direction. \\
        \midrule
        $\mathcal{E}_{e,k,b}^{(h)}$
        & Bit-width-dependent distortion $\mathbb{E}_{\epsilon}\lVert\boldsymbol{\varepsilon}_{e,k,b}^{(h)}\rVert_2^2$. \\
        \midrule
        $\kappa_b$
        & Shared low-bit coefficient for $b\in\{2,3,4\}$. \\
        \midrule
        $\mathcal{L}_{e,P}^{(h)}(b)$
        & Expected routing-weighted output reconstruction loss from quantizing spectral directions with the quantized shared basis fixed. \\
        \midrule
        $\mathcal{L}_{e,k}^{(h)}(b)$
        & Component-wise quantization cost of assigning bit-width $b$. \\
        \midrule
        \multicolumn{2}{c}{\textbf{Global Bit Allocation}} \\
        \midrule
        $b\in\mathcal{B}$
        & Candidate bit-width in $\mathcal{B}=\{16,8,6,4,3,2,1,0\}$. \\
        \midrule
        $y_{\ell,e,k,b}^{(h)}$
        & Binary variable assigning bit-width $b$ to component $(\ell,e,h,k)$; $y$ collects all assignment variables. \\
        \midrule
        $\mathcal{L}_{\mathrm{FQCM}}^{(h)}(y)$
        & Global FQCM loss summed over all layers, experts, and spectral components for projection type $h$. \\
        \midrule
        $\mathfrak{b}_{\mathrm{eq}}$, $B_{\mathrm{alloc}}^{(h)}(\mathfrak{b}_{\mathrm{eq}})$
        & Target bit-width and per-type global normalized bit budget. \\
        \bottomrule
    \end{tabularx}
\end{table}

\clearpage

\section{Details on Models and Evaluation}
\label{app:reproduction_settings}

\begin{table}[h]
    \centering
    \caption{Configurations of the evaluated MoE LLMs. Parameter counts are in billions. The expert configuration is expressed as (routed experts + shared experts) $\times$ MoE layers; Top-K counts the routed experts selected per token. DeepSeek-V2-Lite's first dense layer is excluded from the MoE layer count. Qwen3-30B-A3B has no shared experts.}
    \label{tab:evaluated_models}
    \tablestyle
    \begin{tabular}{@{}cccccc@{}}
        \toprule
        \textbf{Model} & \textbf{Total Params.} & \textbf{Active Params.} & \textbf{Variant} & \textbf{\#Experts} & \textbf{Top-K} \\
        \midrule
        \textsc{DeepSeek-V2-Lite} & 15.7 & 2.4 & Base & $(64+2)\times26$ & 6 \\
        \textsc{Qwen3-30B-A3B} & 30.5 & 3.3 & Base & $128\times48$ & 8 \\
        \textsc{Qwen3-Next-80B-A3B} & 80 & 3 & Instruct & $(512+1)\times48$ & 10 \\
        \bottomrule
    \end{tabular}
\end{table}

We evaluate \textsc{BitsMoE} on three MoE LLMs spanning different parameter scales, training stages, and architectures: DeepSeek-V2-Lite~\citep{liu2024deepseek}, Qwen3-30B-A3B-Base~\citep{qwen3}, and Qwen3-Next-80B-A3B-Instruct~\citep{qwen3next}. Tab.~\ref{tab:evaluated_models} summarizes their parameter counts and expert configurations. The first two checkpoints are pretrained base models, whereas Qwen3-Next-80B-A3B is instruction-tuned. DeepSeek-V2-Lite combines Multi-head Latent Attention (MLA) with routed and shared experts; Qwen3-30B-A3B uses grouped-query attention with routed experts only; and Qwen3-Next-80B-A3B combines Gated DeltaNet and Gated Attention with a high-sparsity MoE architecture containing both routed and shared experts.

We assess both downstream task accuracy and language-modeling perplexity. The six zero-shot tasks are ARC-Easy~\citep{clark2018think}, HellaSwag~\citep{zellers2019hellaswag}, MathQA~\citep{amini2019mathqa}, MMLU~\citep{hendrycks2009measuring}, OpenBookQA~\citep{mihaylov2018can}, and WinoGrande~\citep{sakaguchi2021winogrande}, covering science question answering, commonsense reasoning, mathematical reasoning, and broad knowledge. We additionally evaluate multi-step mathematical reasoning on GSM8K~\citep{cobbe2021training}. All seven downstream tasks are evaluated using lm-evaluation-harness (version 0.4.9.1)~\citep{eval-harness}. We also report WikiText-2~\citep{merity2016pointer} perplexity for Qwen3-30B-A3B.

Accuracy and perplexity are evaluated on NVIDIA A100 GPUs. Average accuracy is computed over the seven downstream tasks. GSM8K is evaluated in a 5-shot setting. Chat templates are applied to instruction-tuned models. WikiText-2 perplexity is evaluated on the test split of \texttt{wikitext-2-raw-v1} using the strided sliding-window strategy described in the \href{https://huggingface.co/docs/transformers/en/perplexity}{HuggingFace Transformers documentation}, with a window length of 2,048 tokens and a stride of 512 tokens. For calibration, we use 1,024 samples from the English C4 training split~\citep{raffel2020exploring}, with a maximum sequence length of 2,048 tokens and sampling seed 42. The C4--MATH calibration mixtures are examined separately in App.~\ref{app:calibration_dataset_composition}. Shared bases are quantized to 8-bit precision using GPTQ; magnitude matrices $A_e^{(h)}$ and stored scales use FP16. Routers, shared experts, and non-expert modules are not quantized.

\section{Algorithm Details}
\label{app:algorithm_details}

Algorithm~\ref{alg:bitsmoe_quantization} collects costs over all $L$ MoE layers before solving three independent global ILPs, one for each projection type. Following the notation table (Tab.~\ref{tab:app_detailed_notation}), we use \(m=d_{\mathrm{out}}\) and \(n=n^{(h)}=d_{\mathrm{in}}\) in the full-component setting. Each type receives one third of the global bit budget \(B\) remaining after reserving the 8-bit shared bases, giving the global normalized bit budget \(B_{\mathrm{alloc}}^{(h)}=\lfloor B/(3m)\rfloor\) derived in App.~\ref{app:equivalent_bit_budget}. Calibration statistics are collected with the original FP16 model.

The shared spectral basis is quantized to 8-bit precision using GPTQ, yielding $\widehat{\boldsymbol{\Phi}}_{\ell}^{(h)}$. The expert-specific spectral directions and component magnitudes are retained from SSD for subsequent bit allocation and quantization.

\begin{algorithm}[p]
    \caption{\textsc{BitsMoE} Offline Quantization}
    \label{alg:bitsmoe_quantization}
    \normalsize
    \setlength{\baselineskip}{14pt}
    \begin{algorithmic}[1]
        \REQUIRE Pretrained model $\mathcal{M}$, calibration data $\mathcal{D}$, target bit-width $\mathfrak{b}_{\mathrm{eq}}$
        \ENSURE Quantized model with packed expert directions and 8-bit GPTQ shared bases
        \STATE $\mathcal{B}\gets\{16,8,6,4,3,2,1,0\}$
        \FOR{each MoE layer $\ell\in[L]$ in $\mathcal{M}$}
            \STATE Collect original projection inputs, routing weights, and unweighted and routing-weighted second moments on $\mathcal{D}$.
            \FOR{each projection type $h\in\mathcal{H}$}
                \STATE Concatenate $\{W_{\ell,e}^{(h)}\}_{e=1}^{E}$ vertically for $\mathtt{gate\_proj}$ and $\mathtt{up\_proj}$, and horizontally for $\mathtt{down\_proj}$.
                \STATE Apply SSD to obtain $\boldsymbol{\Phi}_{\ell}^{(h)}$ and expert factors $\{\widetilde{P}_{\ell,e}^{(h)}\}_{e=1}^{E}$.
                \STATE Normalize the expert factors into directions $P_{\ell,e}^{(h)}$ and magnitude matrices $A_{\ell,e}^{(h)}$, for all $e$.
                \STATE GPTQ-quantize the shared transform to 8-bit precision using its unweighted input statistics, yielding $\widehat{\boldsymbol{\Phi}}_{\ell}^{(h)}$.
                \STATE Compute activation-dependent importance from Eq.~\eqref{eq:app_unified_beta_def}, with the diagonal approximation and normalization described in App.~\ref{app:activation_weighted_loss}.
                \STATE Build and cache the component-wise quantization costs $\{\mathcal{L}_{\ell,e,k}^{(h)}(b)\}_{e,k,b}$ from squared magnitudes, activation-dependent importance, and bit-width-dependent distortion.
                \STATE Cache $P_{\ell,e}^{(h)}$, $A_{\ell,e}^{(h)}$, and the GPTQ shared-basis integers and scales.
            \ENDFOR
        \ENDFOR
        \STATE $B\gets3LEmn\mathfrak{b}_{\mathrm{eq}}-8\cdot3Ln^2$
        \FOR{each projection type $h\in\mathcal{H}$}
            \STATE $B_{\mathrm{alloc}}^{(h)}(\mathfrak{b}_{\mathrm{eq}})\gets\left\lfloor B/(3m)\right\rfloor$
            \STATE Minimize the global FQCM loss in Eq.~\eqref{eq:fqcm_loss} through the ILP in Eq.~\eqref{eq:component_bit_allocation} for this type, obtaining $\{y_{\ell,e,k,b}^{(h)}\}_{\ell,e,k,b}$.
        \ENDFOR
        \FOR{each layer $\ell$, expert $e$, and projection type $h$}
            \STATE $b_{\ell,e,k}^{(h)}\gets\sum_{b\in\mathcal{B}}b\,y_{\ell,e,k,b}^{(h)}$,\quad for all $k\in[n]$
            \STATE Quantize each direction as $Q_{b_{\ell,e,k}^{(h)}}(p_{\ell,e,k}^{(h)})$ using group size 128.\\
                Omit $b=0$ and retain $b=16$ in FP16.
            \STATE Group directions by bit-width; bit-plane pack groups with $1\leq b\leq8$ and store their original indices and scales.
            \STATE Retain $A_{\ell,e}^{(h)}$ in FP16.
        \ENDFOR
        \STATE Pack the 8-bit shared-basis integers and FP16 scales for Marlin W8A16 execution, once per layer and projection type.
        \RETURN Quantized model
    \end{algorithmic}
\end{algorithm}

\newpage
Algorithm~\ref{alg:bitsmoe_forward} describes the forward pass for a single token in an MoE layer. Matrix--vector multiplication with $\widehat{P}_e^{(h)}$ or its transpose uses stored component indices and on-the-fly tile dequantization of the packed spectral directions. Matrix--vector multiplication with the shared bases uses $\operatorname{Marlin}_{\mathrm{W8A16}}(W,a)$, which evaluates $Wa$ with $W$ stored as packed 8-bit integers with FP16 scales and $a$ cast to FP16. Each shared transform is applied once per token: $\mathtt{gate\_proj}$ and $\mathtt{up\_proj}$ before expert computation, and $\mathtt{down\_proj}$ after accumulating the routing-weighted expert contributions.

\begin{algorithm}[H]
    \caption{\textsc{BitsMoE} Forward Pass with Packed Expert Directions}
    \label{alg:bitsmoe_forward}
    \normalsize
    \setlength{\baselineskip}{14pt}
    \renewcommand{\algorithmiccomment}[1]{\hfill$\triangleright$ #1}
    \begin{algorithmic}[1]
        \REQUIRE Token input $x\in\mathbb{R}^{n}$, packed directions $\{\widehat{P}_e^{(h)}\}$, FP16 magnitudes $\{A_e^{(h)}\}$, 8-bit shared bases $\{\widehat{\boldsymbol{\Phi}}^{(h)}\}$ in Marlin format
        \ENSURE MoE output $y\in\mathbb{R}^{n}$
        \STATE Obtain routed experts $\mathcal{T}(x)$ and weights $\{g_e(x)\}$ from the router.
        \FOR{each $h\in\mathcal{H}_{\mathrm{in}}$}
            \STATE $z^{(h)}\gets\operatorname{Marlin}_{\mathrm{W8A16}}\!\left(\bigl(\widehat{\boldsymbol{\Phi}}^{(h)}\bigr)^\top,x\right)$
        \ENDFOR
        \STATE $v\gets0\in\mathbb{R}^{n}$
        \FOR{each $e\in\mathcal{T}(x)$}
            \FOR{each $h\in\mathcal{H}_{\mathrm{in}}$}
                \STATE $u_e^{(h)}\gets\widehat{P}_e^{(h)}A_e^{(h)}z^{(h)}$
            \ENDFOR
            \STATE $a_e\gets\operatorname{SiLU}\bigl(u_e^{(\mathtt{gate\_proj})}\bigr)\odot u_e^{(\mathtt{up\_proj})}$
            \STATE $v\gets v+g_e(x)A_e^{(h_{\mathrm{dn}})}\bigl(\widehat{P}_e^{(h_{\mathrm{dn}})}\bigr)^\top a_e$
        \ENDFOR
        \STATE $y\gets\operatorname{Marlin}_{\mathrm{W8A16}}\!\left(\widehat{\boldsymbol{\Phi}}^{(h_{\mathrm{dn}})},v\right)$
        \STATE Add the original shared-expert branch, including its gate, if present.
        \RETURN $y$
    \end{algorithmic}
\end{algorithm}

\section{Mathematical Formulation}
\label{app:piecewise_ilp_details}

This section develops the mathematical formulation underlying Shared-basis Spectral Decomposition (SSD) and Factorized Quantization Cost Modeling (FQCM) in \textsc{BitsMoE}. We first formalize SSD for all three expert projection types, separating a shared spectral basis, component magnitudes, and expert-specific unit-norm directions that serve as quantization units. Building on this representation, we derive the output reconstruction loss induced by quantizing the expert-specific spectral directions and state the approximations through which FQCM combines intrinsic spectral importance, activation-dependent importance, and bit-width-dependent distortion into component-wise costs. Finally, we specify the piecewise distortion model, translate the target bit-width into an allocation budget, and formulate the resulting component-wise ILP problem. Following the notation table (Tab.~\ref{tab:app_detailed_notation}), $\ell\in[L]$ indexes MoE layers, $e\in[E]$ indexes experts, $k\in[n^{(h)}]$ indexes spectral components, $b\in\mathcal{B}$ denotes a candidate bit-width, and $h\in\mathcal{H}$ indexes projection types. Thus, fully indexed quantities take forms such as $p_{\ell,e,k}^{(h)}$ and $\mathcal{E}_{\ell,e,k,b}^{(h)}$. For concise presentation, we omit the layer index $\ell$ in App.~\ref{app:shared_basis_decomposition} and App.~\ref{app:activation_weighted_loss}, and omit $\ell$, $e$, and $h$ in App.~\ref{app:piecewise_reconstruction_error}.

\subsection{Shared-basis Spectral Decomposition~(SSD)}
\label{app:shared_basis_decomposition}

The SSD construction applies to every MoE layer; throughout App.~\ref{app:shared_basis_decomposition}, we fix an arbitrary layer and omit the layer index $\ell$.

Within an MoE layer, expert weights operate on the same feature space and often exhibit substantial redundancy. Instead of decomposing each expert independently, we jointly decompose all experts to extract a shared spectral basis. We denote the expert projection types by \(\mathcal{H}\coloneqq\{\mathtt{gate\_proj},\mathtt{up\_proj},\mathtt{down\_proj}\}\). We group \(\mathtt{gate\_proj}\) and \(\mathtt{up\_proj}\) into \(\mathcal{H}_{\mathrm{in}}\coloneqq\{\mathtt{gate\_proj},\mathtt{up\_proj}\}\), for which the shared bases act on the input feature space, and denote \(\mathtt{down\_proj}\) by \(h_{\mathrm{dn}}\), for which the shared basis acts on the output feature space. Following Eq.~\eqref{eq:ssd_joint}, $d_{\mathrm{in}}$ and $d_{\mathrm{out}}$ denote the input and output dimensions of the \(\mathtt{gate\_proj}\) and \(\mathtt{up\_proj}\) weights. Thus, $W_e^{(h)}\in\mathbb{R}^{d_{\mathrm{out}}\times d_{\mathrm{in}}}$ for $h\in\mathcal{H}_{\mathrm{in}}$, whereas $W_e^{(h_{\mathrm{dn}})}\in\mathbb{R}^{d_{\mathrm{in}}\times d_{\mathrm{out}}}$.

For \(h\in\mathcal{H}_{\mathrm{in}}\), expert weights share the same input feature space. We concatenate expert weights along the output-channel dimension:
\begin{equation}
    \label{eq:app_w_cat}
    W_{\mathrm{cat}}^{(h)}
    \coloneqq
    \begin{bmatrix}
        W_{1}^{(h)} \\
        \vdots \\
        W_{E}^{(h)}
    \end{bmatrix}.
\end{equation}
We then compute the SVD of the concatenated matrix and merge the singular values into the left factor:
\begin{equation}
    \label{eq:app_cat_svd_absorbed}
    W_{\mathrm{cat}}^{(h)}
    =
    U_{\mathrm{cat}}^{(h)}
    \boldsymbol{\Sigma}^{(h)}
    \bigl(\boldsymbol{\Phi}^{(h)}\bigr)^{\top}
    =
    \widetilde{P}_{\mathrm{cat}}^{(h)}
    \bigl(\boldsymbol{\Phi}^{(h)}\bigr)^{\top},
    \qquad
    \widetilde{P}_{\mathrm{cat}}^{(h)}
    \coloneqq
    U_{\mathrm{cat}}^{(h)}
    \boldsymbol{\Sigma}^{(h)}.
\end{equation}
After merging the singular values, we partition the resulting matrix according to expert blocks:
\begin{equation}
    \label{eq:app_u_tilde_partition}
    \widetilde{P}_{\mathrm{cat}}^{(h)}
    =
    \begin{bmatrix}
        \widetilde{P}_{1}^{(h)} \\
        \vdots \\
        \widetilde{P}_{E}^{(h)}
    \end{bmatrix},
    \qquad
    W_{e}^{(h)}
    =
    \widetilde{P}_{e}^{(h)}
    \bigl(\boldsymbol{\Phi}^{(h)}\bigr)^{\top}.
\end{equation}

For \(h=h_{\mathrm{dn}}\), expert weights share the same output feature space. We concatenate expert weights along the input-channel dimension:
\begin{equation}
    \label{eq:app_down_w_cat}
    W_{\mathrm{cat}}^{(h)}
    \coloneqq
    \begin{bmatrix}
        W_{1}^{(h)}
        &
        \cdots
        &
        W_{E}^{(h)}
    \end{bmatrix}.
\end{equation}
The corresponding SVD-and-absorption step is
\begin{equation}
    \label{eq:app_down_cat_svd_absorbed}
    W_{\mathrm{cat}}^{(h)}
    =
    \boldsymbol{\Phi}^{(h)}
    \boldsymbol{\Sigma}^{(h)}
    \bigl(V_{\mathrm{cat}}^{(h)}\bigr)^{\top}
    =
    \boldsymbol{\Phi}^{(h)}
    \bigl(\widetilde{P}_{\mathrm{cat}}^{(h)}\bigr)^{\top},
    \qquad
    \widetilde{P}_{\mathrm{cat}}^{(h)}
    \coloneqq
    V_{\mathrm{cat}}^{(h)}
    \boldsymbol{\Sigma}^{(h)}.
\end{equation}
Partitioning \(\widetilde{P}_{\mathrm{cat}}^{(h)}\) according to expert input-channel blocks gives
\begin{equation}
    \label{eq:app_down_p_tilde_partition}
    \widetilde{P}_{\mathrm{cat}}^{(h)}
    =
    \begin{bmatrix}
        \widetilde{P}_{1}^{(h)} \\
        \vdots \\
        \widetilde{P}_{E}^{(h)}
    \end{bmatrix},
    \qquad
    W_{e}^{(h)}
    =
    \boldsymbol{\Phi}^{(h)}
    \bigl(\widetilde{P}_{e}^{(h)}\bigr)^{\top}.
\end{equation}

\begin{definition}[Spectral component and magnitude matrix]
\label{def:app_spectral_component_energy}
Let \(\boldsymbol{\phi}_{k}^{(h)}\) denote the \(k\)-th column of the shared basis \(\boldsymbol{\Phi}^{(h)}\), and let \(\widetilde{p}_{e,k}^{(h)}\coloneqq\widetilde{P}_{e}^{(h)}[:,k]\) denote the corresponding expert-specific spectral vector. We define its magnitude as
\begin{equation}
    \label{eq:app_spectral_energy}
    \boldsymbol{\alpha}_{e,k}^{(h)}
    \coloneqq
    \left\|
        \widetilde{p}_{e,k}^{(h)}
    \right\|_2.
\end{equation}
The component magnitudes are collected into a diagonal matrix
\begin{equation}
    \label{eq:app_energy_matrix}
    A_{e}^{(h)}
    \coloneqq
    \operatorname{Diag}
    (\boldsymbol{\alpha}_{e,1}^{(h)},\ldots,\boldsymbol{\alpha}_{e,n^{(h)}}^{(h)}).
\end{equation}
The corresponding rank-one spectral component is
\begin{equation}
    \label{eq:app_spectral_component}
    C_{e,k}^{(h)}
    \coloneqq
    \begin{cases}
        \widetilde{p}_{e,k}^{(h)}\bigl(\boldsymbol{\phi}_{k}^{(h)}\bigr)^{\top},
            & h\in\mathcal{H}_{\mathrm{in}}, \\[3pt]
        \boldsymbol{\phi}_{k}^{(h)}\bigl(\widetilde{p}_{e,k}^{(h)}\bigr)^{\top},
            & h=h_{\mathrm{dn}}.
    \end{cases}
\end{equation}
\end{definition}

\begin{definition}[Normalized expert-specific spectral matrix]
\label{def:app_normalized_spectral_matrix}
To decouple spectral component magnitude from direction, each column of \(\widetilde{P}_{e}^{(h)}\) is normalized by its magnitude:
\begin{equation}
    \label{eq:app_normalized_spectral_matrix}
    P_{e}^{(h)}
    \coloneqq
    [p_{e,1}^{(h)},\ldots,p_{e,n^{(h)}}^{(h)}],
    \qquad
    p_{e,k}^{(h)}
    \coloneqq
    \frac{\widetilde{p}_{e,k}^{(h)}}{\boldsymbol{\alpha}_{e,k}^{(h)}}.
\end{equation}
Each expert-specific spectral direction satisfies \(p_{e,k}^{(h)}\in\mathbb{R}^{m}\), where \(m=d_{\mathrm{out}}\) denotes its dimension for all three projection types. Thus, \(\|p_{e,k}^{(h)}\|_2=1\) for every spectral component. The expert weight admits the unified normalized shared-basis form
\begin{equation}
    \label{eq:app_normalized_decomposition}
    W_{e}^{(h)}
    =
    \begin{cases}
        P_{e}^{(h)}A_{e}^{(h)}\bigl(\boldsymbol{\Phi}^{(h)}\bigr)^{\top}
        =
        \sum_{k=1}^{n^{(h)}}\boldsymbol{\alpha}_{e,k}^{(h)}p_{e,k}^{(h)}\bigl(\boldsymbol{\phi}_{k}^{(h)}\bigr)^{\top},
            & h\in\mathcal{H}_{\mathrm{in}}, \\[5pt]
        \boldsymbol{\Phi}^{(h)}A_{e}^{(h)}\bigl(P_{e}^{(h)}\bigr)^{\top}
        =
        \sum_{k=1}^{n^{(h)}}\boldsymbol{\alpha}_{e,k}^{(h)}\boldsymbol{\phi}_{k}^{(h)}\bigl(p_{e,k}^{(h)}\bigr)^{\top},
            & h=h_{\mathrm{dn}}.
    \end{cases}
\end{equation}
\end{definition}

For each projection type $h\in\mathcal{H}=\{\mathtt{gate\_proj},\mathtt{up\_proj},\mathtt{down\_proj}\}$, the number of spectral components is $n^{(h)}=\min(E d_{\mathrm{out}},d_{\mathrm{in}})$, with $\boldsymbol{\Phi}^{(h)}\in\mathbb{R}^{d_{\mathrm{in}}\times n^{(h)}}$ and $P_e^{(h)}\in\mathbb{R}^{d_{\mathrm{out}}\times n^{(h)}}$.

\paragraph{Spectral energy.}
Following Eq.~\eqref{eq:c2_norm}, we define the spectral energy of spectral component \(C_{e,k}^{(h)}\) as its squared Frobenius norm:
\begin{equation}
    \label{eq:app_component_spectral_energy}
    \begin{aligned}
        \left\lVert C_{e,k}^{(h)}\right\rVert_F^2
        &= \operatorname{Tr}\!\left[\bigl(C_{e,k}^{(h)}\bigr)^\top C_{e,k}^{(h)}\right] \\
        &= \bigl(\boldsymbol{\alpha}_{e,k}^{(h)}\bigr)^2
        \left\lVert p_{e,k}^{(h)}\right\rVert_2^2
        \left\lVert\boldsymbol{\phi}_k^{(h)}\right\rVert_2^2
        = \bigl(\boldsymbol{\alpha}_{e,k}^{(h)}\bigr)^2.
    \end{aligned}
\end{equation}
The last equality follows from the unit norms of the expert-specific direction \(p_{e,k}^{(h)}\) and the shared basis direction \(\boldsymbol{\phi}_k^{(h)}\). This identity holds for every \(h\in\mathcal{H}\).

Each spectral component combines a shared basis direction \(\boldsymbol{\phi}_k^{(h)}\) with an expert-specific magnitude \(\boldsymbol{\alpha}_{e,k}^{(h)}\) and normalized spectral direction \(p_{e,k}^{(h)}\). We therefore use the expert-specific spectral direction \(p_{e,k}^{(h)}\) of each component as the quantization unit in \textsc{BitsMoE}, assigning a bit-width to each direction.

\subsection{Factorized Quantization Cost Modeling~(FQCM)}
\label{app:activation_weighted_loss}

The FQCM derivation applies to every MoE layer; throughout App.~\ref{app:activation_weighted_loss}, we fix an arbitrary layer and omit the layer index $\ell$.

For each projection type $h\in\mathcal{H}$, we first quantize the shared basis to 8-bit precision using GPTQ~\citep{frantar-gptq} and denote the resulting basis by $\widehat{\boldsymbol{\Phi}}^{(h)}$.

For expert $e\in[E]$ and projection type \(h\in\mathcal{H}_{\mathrm{in}}\), we apply $b$-bit quantization to the expert-specific spectral directions obtained by SSD:
\begin{equation}
    \label{eq:app_quantized_vector_error}
    \widehat{p}_{e,k}^{(h)}
    =
    Q_b(p_{e,k}^{(h)}),
    \qquad
    \boldsymbol{\varepsilon}_{e,k,b}^{(h)}
    \coloneqq
    p_{e,k}^{(h)}
    -
    \widehat{p}_{e,k}^{(h)}.
\end{equation}
Using the quantized shared basis $\widehat{\boldsymbol{\Phi}}^{(h)}=[\widehat{\boldsymbol{\phi}}_1^{(h)},\ldots,\widehat{\boldsymbol{\phi}}_{n^{(h)}}^{(h)}]$, quantizing the expert-specific spectral directions gives the reconstructed weight
\begin{equation}
    \label{eq:app_reconstructed_weight}
    \widehat{W}_{e}^{(h)}
    =
    \sum_{k=1}^{n^{(h)}}
    \boldsymbol{\alpha}_{e,k}^{(h)}
    \widehat{p}_{e,k}^{(h)}
    \bigl(\widehat{\boldsymbol{\phi}}_{k}^{(h)}\bigr)^{\top}.
\end{equation}
Let $W_{e,q\phi}^{(h)}\coloneqq P_e^{(h)}A_e^{(h)}\bigl(\widehat{\boldsymbol{\Phi}}^{(h)}\bigr)^\top$. We measure the additional error introduced by quantizing the expert-specific spectral directions relative to $W_{e,q\phi}^{(h)}$ and define the corresponding weight perturbation as
\begin{equation}
    \label{eq:app_direction_weight_perturbation}
    \boldsymbol{\Delta}_{e}^{(h)}
    \coloneqq
    W_{e,q\phi}^{(h)}-\widehat{W}_{e}^{(h)}
    =
    \sum_{k=1}^{n^{(h)}}
    \boldsymbol{\alpha}_{e,k}^{(h)}
    \boldsymbol{\varepsilon}_{e,k,b}^{(h)}
    \bigl(\widehat{\boldsymbol{\phi}}_{k}^{(h)}\bigr)^{\top}.
\end{equation}
We define the quantization error matrix of the expert-specific spectral directions as \(\boldsymbol{E}_{P,e}^{(h)}\coloneqq P_{e}^{(h)}-\widehat{P}_{e}^{(h)}=[\boldsymbol{\varepsilon}_{e,1,b}^{(h)},\ldots,\boldsymbol{\varepsilon}_{e,n^{(h)},b}^{(h)}]\). The weight perturbation can then be written as \(\boldsymbol{\Delta}_{e}^{(h)}=\boldsymbol{E}_{P,e}^{(h)}A_{e}^{(h)}\bigl(\widehat{\boldsymbol{\Phi}}^{(h)}\bigr)^{\top}\).

\begin{definition}[Routing-weighted Hessian]
\label{def:app_routing_guided_hessian}
For expert $e$ and projection type $h\in\mathcal{H}_{\mathrm{in}}$, let $X_e^{(h)}=[x_{e,1}^{(h)},\ldots,x_{e,t}^{(h)}]\in\mathbb{R}^{d_{\mathrm{in}}\times t}$ contain the $t$ token activations routed to this expert, and let $g_e=[g_{e,1},\ldots,g_{e,t}]^\top\in\mathbb{R}^{t}$ contain the corresponding routing weights. With $G_e\coloneqq\operatorname{Diag}(g_e)$, we define the routing-weighted Hessian as
\begin{equation}
    \label{eq:app_affinity_gram}
    H_{g,e}^{(h)}
    \coloneqq
    X_e^{(h)}G_e\bigl(X_e^{(h)}\bigr)^\top
    =
    \sum_{j=1}^{t}
    g_{e,j}
    x_{e,j}^{(h)}
    \bigl(x_{e,j}^{(h)}\bigr)^\top.
\end{equation}
\end{definition}
Since $g_{e,j}\geq0$, $H_{g,e}^{(h)}\succeq0$. Writing $X_{g,e}^{(h)}\coloneqq X_e^{(h)}G_e^{1/2}$ gives $H_{g,e}^{(h)}=X_{g,e}^{(h)}\bigl(X_{g,e}^{(h)}\bigr)^\top$.

Under this metric, we define the expected routing-weighted output reconstruction loss from quantizing the spectral directions as
\begin{equation}
    \label{eq:app_activation_output_loss}
    \mathcal{L}_{e,P}^{(h)}(b)
    \coloneqq
    \mathbb{E}_{\epsilon}
    \left\|
        (W_{e,q\phi}^{(h)}-\widehat{W}_{e}^{(h)})
        X_{g,e}^{(h)}
    \right\|_F^2,
\end{equation}
where $\mathbb{E}_{\epsilon}$ denotes expectation over the quantization errors. The subscript $P$ identifies the quantization of expert-specific spectral directions with the quantized shared basis fixed, as in Eq.~\eqref{eq:activation_output_objective_process}; the dependence of $\widehat{W}_e^{(h)}$ and $\widehat{P}_e^{(h)}$ on $b$ is implicit.

\paragraph{Factorized reconstruction-loss approximation.}
For expert $e$ and projection type $h\in\mathcal{H}_{\mathrm{in}}$, the expected routing-weighted output reconstruction loss admits the following additive approximation when cross-component quantization-error terms are neglected:
\begin{equation}
    \label{eq:app_diagonal_loss}
    \mathcal{L}_{e,P}^{(h)}(b)
    \approx
    \sum_{k=1}^{n^{(h)}}
    \bigl(\boldsymbol{\alpha}_{e,k}^{(h)}\bigr)^2
    \boldsymbol{\beta}_{e,k}^{(h)}
    \mathcal{E}_{e,k,b}^{(h)},
\end{equation}
where the activation-dependent importance and bit-width-dependent distortion are
\begin{equation}
    \label{eq:app_beta_def}
    \begin{aligned}
    \boldsymbol{\beta}_{e,k}^{(h)}
    &\coloneqq
    \bigl(\widehat{\boldsymbol{\phi}}_k^{(h)}\bigr)^\top
    H_{g,e}^{(h)}
    \widehat{\boldsymbol{\phi}}_k^{(h)},
    \\
    \mathcal{E}_{e,k,b}^{(h)}
    &\coloneqq
    \mathbb{E}_{\epsilon}\left\|\boldsymbol{\varepsilon}_{e,k,b}^{(h)}\right\|_2^2.
    \end{aligned}
\end{equation}
\noindent\textit{Derivation.}
Starting from Eq.~\eqref{eq:app_activation_output_loss} and substituting the perturbation in Eq.~\eqref{eq:app_direction_weight_perturbation}, we obtain
\begin{equation}
    \label{eq:app_full_spectrum_error}
    \begin{aligned}
    \mathcal{L}_{e,P}^{(h)}(b)
    &=\mathbb{E}_{\epsilon}\left\|
        \bigl(W_{e,q\phi}^{(h)}-\widehat{W}_{e}^{(h)}\bigr)X_{g,e}^{(h)}
    \right\|_F^2 \\
    &=\mathbb{E}_{\epsilon}\left\|
        \boldsymbol{\Delta}_{e}^{(h)}X_{g,e}^{(h)}
    \right\|_F^2 \\
    &=\mathbb{E}_{\epsilon}\operatorname{Tr}\!\left[
        \boldsymbol{\Delta}_{e}^{(h)}X_{g,e}^{(h)}
        \bigl(X_{g,e}^{(h)}\bigr)^\top
        \bigl(\boldsymbol{\Delta}_{e}^{(h)}\bigr)^\top
    \right] \\
    &=\mathbb{E}_{\epsilon}\operatorname{Tr}\!\left[
        \boldsymbol{\Delta}_{e}^{(h)}H_{g,e}^{(h)}
        \bigl(\boldsymbol{\Delta}_{e}^{(h)}\bigr)^\top
    \right] \\
    &=\sum_{k=1}^{n^{(h)}}\sum_{l=1}^{n^{(h)}}
    \boldsymbol{\alpha}_{e,k}^{(h)}\boldsymbol{\alpha}_{e,l}^{(h)}
    \mathbb{E}_{\epsilon}\operatorname{Tr}\!\left[
        \boldsymbol{\varepsilon}_{e,k,b}^{(h)}
        \bigl(\widehat{\boldsymbol{\phi}}_k^{(h)}\bigr)^\top
        H_{g,e}^{(h)}\widehat{\boldsymbol{\phi}}_l^{(h)}
        \bigl(\boldsymbol{\varepsilon}_{e,l,b}^{(h)}\bigr)^\top
    \right] \\
    &=\sum_{k=1}^{n^{(h)}}\sum_{l=1}^{n^{(h)}}
    \boldsymbol{\alpha}_{e,k}^{(h)}\boldsymbol{\alpha}_{e,l}^{(h)}
    \left(
        \bigl(\widehat{\boldsymbol{\phi}}_k^{(h)}\bigr)^\top
        H_{g,e}^{(h)}\widehat{\boldsymbol{\phi}}_l^{(h)}
    \right)
    \mathbb{E}_{\epsilon}\!\left[
        \bigl(\boldsymbol{\varepsilon}_{e,k,b}^{(h)}\bigr)^\top
        \boldsymbol{\varepsilon}_{e,l,b}^{(h)}
    \right].
    \end{aligned}
\end{equation}

For $b\geq1$, we model the quantization errors as symmetric about zero, giving $\mathbb{E}_{\epsilon}[\boldsymbol{\varepsilon}_{e,k,b}^{(h)}]=0$. Assuming that quantization errors across distinct spectral components are approximately uncorrelated, we obtain
\begin{equation}
    \label{eq:app_diagonal_error}
    \mathbb{E}_{\epsilon}\!\left[
        \bigl(\boldsymbol{\varepsilon}_{e,k,b}^{(h)}\bigr)^\top
        \boldsymbol{\varepsilon}_{e,l,b}^{(h)}
    \right]
    \approx
    \mathbb{E}_{\epsilon}\!\left[\boldsymbol{\varepsilon}_{e,k,b}^{(h)}\right]^\top
    \mathbb{E}_{\epsilon}\!\left[\boldsymbol{\varepsilon}_{e,l,b}^{(h)}\right]
    =0,
    \qquad k\ne l.
\end{equation}
For zero-bit eviction ($b=0$), where $\boldsymbol{\varepsilon}_{e,k,0}^{(h)}=p_{e,k}^{(h)}$, we retain the same diagonal approximation in the allocation objective.

Separating the diagonal and cross-component terms in Eq.~\eqref{eq:app_full_spectrum_error} gives
\begin{equation}
    \label{eq:app_diagonal_loss_derivation}
    \begin{aligned}
    \mathcal{L}_{e,P}^{(h)}(b)
    &=\sum_{k=1}^{n^{(h)}}
        \bigl(\boldsymbol{\alpha}_{e,k}^{(h)}\bigr)^2
        \boldsymbol{\beta}_{e,k}^{(h)}\mathcal{E}_{e,k,b}^{(h)}
    \\
    &\quad+\sum_{\substack{k,l=1\\k\ne l}}^{n^{(h)}}
        \boldsymbol{\alpha}_{e,k}^{(h)}\boldsymbol{\alpha}_{e,l}^{(h)}
        \left(
            \bigl(\widehat{\boldsymbol{\phi}}_k^{(h)}\bigr)^\top
            H_{g,e}^{(h)}\widehat{\boldsymbol{\phi}}_l^{(h)}
        \right)
        \mathbb{E}_{\epsilon}\!\left[
            \bigl(\boldsymbol{\varepsilon}_{e,k,b}^{(h)}\bigr)^\top
            \boldsymbol{\varepsilon}_{e,l,b}^{(h)}
        \right]
    \\
    &\approx\sum_{k=1}^{n^{(h)}}
        \bigl(\boldsymbol{\alpha}_{e,k}^{(h)}\bigr)^2
        \boldsymbol{\beta}_{e,k}^{(h)}\mathcal{E}_{e,k,b}^{(h)}.
    \end{aligned}
\end{equation}
Only the final step invokes the diagonal approximation, yielding Eq.~\eqref{eq:app_diagonal_loss} and the factorization in Eq.~\eqref{eq:activation_output_objective_formula}. Since $H_{g,e}^{(h)}\succeq0$, $\boldsymbol{\beta}_{e,k}^{(h)}\geq0$.

For $h=h_{\mathrm{dn}}$, the input activations satisfy $X_e^{(h)}\in\mathbb{R}^{d_{\mathrm{out}}\times t}$, and the quantized shared basis $\widehat{\boldsymbol{\Phi}}^{(h)}$ acts in the output space. Define $W_{e,q\phi}^{(h)}\coloneqq\widehat{\boldsymbol{\Phi}}^{(h)}A_e^{(h)}\bigl(P_e^{(h)}\bigr)^\top$ and $\widehat W_e^{(h)}\coloneqq\widehat{\boldsymbol{\Phi}}^{(h)}A_e^{(h)}\bigl(\widehat P_e^{(h)}\bigr)^\top$. The corresponding weight perturbation from quantizing the expert-specific spectral directions is
\begin{equation}
    \label{eq:app_down_perturbation}
    \begin{aligned}
    \boldsymbol{\Delta}_e^{(h)}
    &\coloneqq W_{e,q\phi}^{(h)}-\widehat W_e^{(h)}
    =\widehat{\boldsymbol{\Phi}}^{(h)}A_e^{(h)}
    \bigl(\boldsymbol E_{P,e}^{(h)}\bigr)^\top \\
    &=\sum_{k=1}^{n^{(h)}}\boldsymbol{\alpha}_{e,k}^{(h)}
    \widehat{\boldsymbol{\phi}}_k^{(h)}
    \bigl(\boldsymbol{\varepsilon}_{e,k,b}^{(h)}\bigr)^\top.
    \end{aligned}
\end{equation}
Using $H_{g,e}^{(h)}=X_e^{(h)}G_e\bigl(X_e^{(h)}\bigr)^\top\in\mathbb{R}^{d_{\mathrm{out}}\times d_{\mathrm{out}}}$, define the shared-basis Gram matrix $M^{(h)}\coloneqq\bigl(\widehat{\boldsymbol{\Phi}}^{(h)}\bigr)^\top\widehat{\boldsymbol{\Phi}}^{(h)}$, with entries $M_{kl}^{(h)}=\bigl(\widehat{\boldsymbol{\phi}}_k^{(h)}\bigr)^\top\widehat{\boldsymbol{\phi}}_l^{(h)}$. The expected routing-weighted output reconstruction loss expands as
\begin{equation}
    \label{eq:app_down_exact_metric}
    \begin{aligned}
    \mathcal{L}_{e,P}^{(h)}(b)
    &=\mathbb E_{\epsilon}\left\|\boldsymbol{\Delta}_e^{(h)}X_{g,e}^{(h)}\right\|_F^2 \\
    &=\mathbb E_{\epsilon}\operatorname{Tr}\!\left[
      \widehat{\boldsymbol{\Phi}}^{(h)}A_e^{(h)}
      \bigl(\boldsymbol E_{P,e}^{(h)}\bigr)^\top H_{g,e}^{(h)}
      \boldsymbol E_{P,e}^{(h)}A_e^{(h)}
      \bigl(\widehat{\boldsymbol{\Phi}}^{(h)}\bigr)^\top
    \right] \\
    &=\mathbb E_{\epsilon}\operatorname{Tr}\!\left[
      M^{(h)}A_e^{(h)}\bigl(\boldsymbol E_{P,e}^{(h)}\bigr)^\top
      H_{g,e}^{(h)}\boldsymbol E_{P,e}^{(h)}A_e^{(h)}
    \right] \\
    &=\sum_{k=1}^{n^{(h)}}\sum_{l=1}^{n^{(h)}}
      \boldsymbol{\alpha}_{e,k}^{(h)}\boldsymbol{\alpha}_{e,l}^{(h)}M_{kl}^{(h)}
      \mathbb E_{\epsilon}\!\left[
        \bigl(\boldsymbol{\varepsilon}_{e,k,b}^{(h)}\bigr)^\top
        H_{g,e}^{(h)}\boldsymbol{\varepsilon}_{e,l,b}^{(h)}
      \right].
    \end{aligned}
\end{equation}
We assume that the shared-basis Gram matrix satisfies
\begin{equation}
    \label{eq:app_quantized_basis_orthonormality}
    \left\|M^{(h)}-I\right\|_2
    =\left\|
        \bigl(\widehat{\boldsymbol{\Phi}}^{(h)}\bigr)^\top
        \widehat{\boldsymbol{\Phi}}^{(h)}-I
    \right\|_2
    \leq\delta_{\mathrm{orth}}^{(h)},
    \qquad 0\leq\delta_{\mathrm{orth}}^{(h)}\ll1,
\end{equation}
where $\|\cdot\|_2$ denotes the spectral norm and $\delta_{\mathrm{orth}}^{(h)}$ bounds the deviation from orthonormality.

Under this assumption, substituting $M^{(h)}\approx I$ into Eq.~\eqref{eq:app_down_exact_metric} yields
\begin{equation}
    \label{eq:app_down_identity_metric_approx}
    \mathcal{L}_{e,P}^{(h)}(b)
    \approx\sum_{k=1}^{n^{(h)}}
      \bigl(\boldsymbol{\alpha}_{e,k}^{(h)}\bigr)^2
      \mathbb E_{\epsilon}\!\left[
        \bigl(\boldsymbol{\varepsilon}_{e,k,b}^{(h)}\bigr)^\top
        H_{g,e}^{(h)}\boldsymbol{\varepsilon}_{e,k,b}^{(h)}
      \right].
\end{equation}

In Eq.~\eqref{eq:app_down_identity_metric_approx}, the activation-weighted quantization error depends on the error direction and hence on the chosen bit-width. To decouple activation-dependent importance from quantization, we use the original expert-specific spectral direction $p_{e,k}^{(h)}$ obtained from SSD to measure the component's importance under the calibration activation patterns. For nonzero expected distortion, this modeling choice gives the surrogate
\begin{equation}
    \label{eq:app_down_sensitivity_assumption}
    \frac{
      \mathbb E_{\epsilon}\!\left[
        \bigl(\boldsymbol{\varepsilon}_{e,k,b}^{(h)}\bigr)^\top
        H_{g,e}^{(h)}\boldsymbol{\varepsilon}_{e,k,b}^{(h)}
      \right]
    }{
      \mathbb E_{\epsilon}\left\|\boldsymbol{\varepsilon}_{e,k,b}^{(h)}\right\|_2^2
    }
    \approx
    \bigl(p_{e,k}^{(h)}\bigr)^\top H_{g,e}^{(h)}p_{e,k}^{(h)}.
\end{equation}
This choice yields the component-wise cost approximation
\begin{equation}
    \label{eq:app_down_beta_approx}
    \begin{aligned}
    \mathbb E_{\epsilon}\!\left[
      \bigl(\boldsymbol{\varepsilon}_{e,k,b}^{(h)}\bigr)^\top
      H_{g,e}^{(h)}\boldsymbol{\varepsilon}_{e,k,b}^{(h)}
    \right]
    &\approx\boldsymbol{\beta}_{e,k}^{(h)}
      \mathbb E_{\epsilon}\left\|\boldsymbol{\varepsilon}_{e,k,b}^{(h)}\right\|_2^2, \\
    \boldsymbol{\beta}_{e,k}^{(h)}
    &\coloneqq\bigl(p_{e,k}^{(h)}\bigr)^\top H_{g,e}^{(h)}p_{e,k}^{(h)}.
    \end{aligned}
\end{equation}
For $b=0$, this component-wise relation is exact because $\boldsymbol{\varepsilon}_{e,k,0}^{(h)}=p_{e,k}^{(h)}$ and $\|p_{e,k}^{(h)}\|_2=1$.

The activation-dependent importance used in the FQCM model is thus
\begin{equation}
    \label{eq:app_unified_beta_def}
    \boldsymbol{\beta}_{e,k}^{(h)}\coloneqq
    \begin{cases}
      \bigl(\widehat{\boldsymbol{\phi}}_k^{(h)}\bigr)^\top
      H_{g,e}^{(h)}\widehat{\boldsymbol{\phi}}_k^{(h)},
      & h\in\mathcal H_{\mathrm{in}}, \\[3pt]
      \bigl(p_{e,k}^{(h)}\bigr)^\top H_{g,e}^{(h)}p_{e,k}^{(h)},
      & h=h_{\mathrm{dn}}.
    \end{cases}
\end{equation}

\paragraph{Activation-dependent importance for allocation.}
For bit allocation, we evaluate Eq.~\eqref{eq:app_unified_beta_def} using only the diagonal of the routing-weighted Hessian and normalize the result by the mean diagonal of the unweighted activation second moment $X_e^{(h)}\bigl(X_e^{(h)}\bigr)^\top$. This normalization removes differences in overall activation scale across layers, making the activation-dependent importance more comparable for global allocation while retaining expert-specific routing weights through $G_e$. The diagonal approximation retains routing-weighted per-coordinate second moments and omits cross-coordinate correlations, aiming to reduce sensitivity to noise in correlation estimates from limited calibration data.

Using this activation-dependent importance, the component-wise quantization cost of assigning bit-width $b$ to $p_{e,k}^{(h)}$ combines the three factors in Eq.~\eqref{eq:component_cost}:
\begin{equation}
    \label{eq:app_unified_projection_loss}
    \mathcal{L}_{e,k}^{(h)}(b)
    \approx
    \underbrace{\bigl(\boldsymbol{\alpha}_{e,k}^{(h)}\bigr)^2}_{\substack{\text{Intrinsic spectral}\\\text{importance}}}
    \;\cdot\;
    \underbrace{\boldsymbol{\beta}_{e,k}^{(h)}}_{\substack{\text{Activation-dependent}\\\text{importance}}}
    \;\cdot\;
    \underbrace{\mathcal E_{e,k,b}^{(h)}}_{\substack{\text{Bit-width-dependent}\\\text{distortion}}}.
\end{equation}
The resulting sum of component-wise quantization costs is a normalized allocation surrogate; $\mathcal{L}_{e,P}^{(h)}(b)$ in the preceding derivation retains its original reconstruction-loss scale. Here $\mathcal E_{e,k,b}^{(h)}=\mathbb E_{\epsilon}\|\boldsymbol{\varepsilon}_{e,k,b}^{(h)}\|_2^2$ denotes the bit-width-dependent distortion. The following subsection specifies this factor for each candidate bit-width $b\in\mathcal B$.

\subsection{Derivation of Bit-width-dependent Distortion}
\label{app:piecewise_reconstruction_error}

We derive bit-width-dependent distortion using a single-scale reference quantizer and use the resulting expressions as allocation surrogates for group-wise quantization.

The distortion derivation has the same form for every layer, expert, and projection type. Throughout App.~\ref{app:piecewise_reconstruction_error}, we therefore omit the indices $\ell$, $e$, and $h$, writing $p_k$ and $\mathcal{E}_{k,b}$ for $p_{\ell,e,k}^{(h)}$ and $\mathcal{E}_{\ell,e,k,b}^{(h)}$, respectively, and applying the same convention to related quantities.

We now specify the bit-width-dependent distortion \(\mathcal{E}_{k,b}\) incurred by quantizing the expert-specific spectral direction \(p_k\in\mathbb{R}^{m}\), with \(\|p_k\|_2=1\), at each candidate bit-width \(b\in\mathcal{B}=\{16,8,6,4,3,2,1,0\}\). The candidate \(b=16\) denotes an FP16 expert-specific spectral direction. Retaining the expert-specific spectral direction in FP16 incurs no additional quantization error, so $\mathcal{E}_{k,16}=0$. As defined in Eq.~\eqref{eq:app_beta_def}, the bit-width-dependent distortion is
\begin{equation}
    \label{eq:app_distortion_def}
    \mathcal{E}_{k,b}
    \coloneqq
    \mathbb{E}_{\epsilon}
    \left\|
        \boldsymbol{\varepsilon}_{k,b}
    \right\|_2^2.
\end{equation}

\paragraph{Bit-width-dependent distortion under the high-resolution approximation.}
For \(b\in\{6,8\}\), define
\begin{equation}
    \label{eq:app_high_bit_rho_eta}
    \rho_{k}
    \coloneqq
    \|p_{k}\|_{\infty},
    \qquad
    \eta_{k}
    \coloneqq
    \frac{m\rho_{k}^{2}}{3}.
\end{equation}
Symmetric uniform quantization is applied elementwise with a single quantization scale \(s_{k,b}\) for all coordinates of \(p_k\). For each coordinate \(j\in\{1,\ldots,m\}\), define
\begin{equation}
    \label{eq:app_scalar_quantization_uk}
    \begin{aligned}
    Q_b(p_{k,j};s_{k,b})
    \coloneqq{}&
    s_{k,b}\cdot
    \operatorname{clamp}
    \left(
        \left\lfloor
            \frac{p_{k,j}}{s_{k,b}}
        \right\rceil,
        q_{\min},
        q_{\max}
    \right), \\
    q_{\max}={}&2^{b-1}-1,
    \qquad
    q_{\min}=-2^{b-1}.
    \end{aligned}
\end{equation}
Here, $\lfloor z\rceil$ denotes rounding $z$ to the nearest integer.

The coordinate-wise quantization error is
\begin{equation}
    \label{eq:app_coordinate_quant_error_uk}
    \varepsilon_{k,j,b}
    \coloneqq
    p_{k,j}-Q_b(p_{k,j};s_{k,b}).
\end{equation}
Accordingly, the vector-level quantization error is
\begin{equation}
    \label{eq:app_vector_quant_error_uk}
    \boldsymbol{\varepsilon}_{k,b}
    \coloneqq
    p_{k}-Q_b(p_{k};s_{k,b})
    =
    \left[
        \varepsilon_{k,1,b},
        \ldots,
        \varepsilon_{k,m,b}
    \right]^{\top}.
\end{equation}
In the high-resolution regime, clipping is negligible and each scalar rounding error is approximated as uniformly distributed on \([-s_{k,b}/2,s_{k,b}/2)\). Therefore, for each coordinate \(j\),
\begin{equation}
    \label{eq:app_coordinate_uniform_error_uk}
    \mathbb{E}_{\epsilon}\!\left[\varepsilon_{k,j,b}\right]=0,
    \qquad
    \mathbb{E}_{\epsilon}\!\left[\varepsilon_{k,j,b}^{2}\right]
    =
    \frac{1}{s_{k,b}}
    \int_{-s_{k,b}/2}^{s_{k,b}/2}
    \varepsilon^{2}d\varepsilon
    =
    \frac{s_{k,b}^{2}}{12}.
\end{equation}
Since the vector-level squared error is the sum of coordinate-wise squared errors, we have
\begin{equation}
    \label{eq:app_vector_error_sum_uk}
    \begin{aligned}
    \mathbb{E}_{\epsilon}\!\left[
        \left\|
            \boldsymbol{\varepsilon}_{k,b}
        \right\|_2^2
    \right]
    &=
    \mathbb{E}_{\epsilon}\!\left[
        \sum_{j=1}^{m}
        \varepsilon_{k,j,b}^{2}
    \right] \\
    &=
    \sum_{j=1}^{m}
    \mathbb{E}_{\epsilon}\!\left[
        \varepsilon_{k,j,b}^{2}
    \right]
    =
    \frac{m s_{k,b}^{2}}{12}.
    \end{aligned}
\end{equation}
By the definition of \(\rho_{k}\) in Eq.~\eqref{eq:app_high_bit_rho_eta}, the quantization scale is
\begin{equation}
    \label{eq:app_uk_scale}
    s_{k,b}
    =
    \frac{\rho_{k}}{q_{\max}}.
\end{equation}
Substituting Eq.~\eqref{eq:app_uk_scale} into Eq.~\eqref{eq:app_vector_error_sum_uk} gives
\begin{equation}
    \label{eq:app_high_bit_vector_error_uk}
    \mathbb{E}_{\epsilon}\!\left[
        \left\|
            \boldsymbol{\varepsilon}_{k,b}
        \right\|_2^2
    \right]
    =
    \frac{m\rho_{k}^{2}}{12q_{\max}^{2}}.
\end{equation}
For sufficiently large bit-widths, \(q_{\max}=2^{b-1}-1\approx2^{b-1}\). With \(\lambda\coloneqq2\ln2\) and \(\eta_k\) defined in Eq.~\eqref{eq:app_high_bit_rho_eta}, we obtain
\begin{equation}
    \label{eq:app_high_bit_distortion}
    \begin{aligned}
    \mathcal{E}_{k,b}
    &=
    \mathbb{E}_{\epsilon}\!\left[
        \left\|
            \boldsymbol{\varepsilon}_{k,b}
        \right\|_2^2
    \right]
    \approx
    \frac{m\rho_{k}^{2}}{12\cdot2^{2(b-1)}} \\
    &=
    \frac{m\rho_{k}^{2}}{3}
    \exp(-2b\ln2)
    =
    \eta_k\exp(-\lambda b).
    \end{aligned}
\end{equation}

\paragraph{Empirical bit-width-dependent distortion at low bit-widths.}
For \(b\in\{2,3,4\}\), we use the symmetric uniform quantizer
\begin{equation}
    \label{eq:app_lowbit_quantizer}
    Q_b(p_{k,j};s_{k,b})
    \coloneqq
    s_{k,b}
    \cdot
    \operatorname{clamp}
    \left(
        \left\lfloor
            \frac{p_{k,j}}{s_{k,b}}
        \right\rceil,
        q_{\min},
        q_{\max}
    \right),
\end{equation}
where \(q_{\max}=2^{b-1}-1\) and \(q_{\min}=-2^{b-1}\). The vector quantizer \(Q_b(p_{k};s_{k,b})\) is applied elementwise. For a given scale \(s_{k,b}\), the empirical bit-width-dependent distortion is
\begin{equation}
    \label{eq:app_lowbit_empirical_error}
    \mathcal{E}_{k,b}(s_{k,b})
    \coloneqq
    \left\|
        p_{k}
        -
        Q_b(p_{k};s_{k,b})
    \right\|_2^2
    =
    \sum_{j=1}^{m}
    \left(
        p_{k,j}
        -
        Q_b(p_{k,j};s_{k,b})
    \right)^2 .
\end{equation}
For each pair of expert-specific spectral direction and bit-width, we choose the MSE-optimal quantization scale
\begin{equation}
    \label{eq:app_low_bit_opt_scale}
    s_{k,b}^{\ast}
    \in
    \arg\min_{s_{k,b}>0}
    \left\|
        p_{k}
        -
        Q_b(p_{k};s_{k,b})
    \right\|_2^2 .
\end{equation}
At the selected scale \(s_{k,b}^{\ast}\), define the coordinate-wise quantization error for \(j\in\{1,\ldots,m\}\) as
\begin{equation}
    \label{eq:app_lowbit_coordinate_error}
    \varepsilon_{k,j,b}
    \coloneqq
    p_{k,j}
    -
    Q_b(p_{k,j};s_{k,b}^{\ast}) .
\end{equation}
Evaluating the distortion at this scale gives
\begin{equation}
    \label{eq:app_lowbit_empirical_error_opt}
    \mathcal{E}_{k,b}
    =
    \mathcal{E}_{k,b}(s_{k,b}^{\ast})
    =
    \sum_{j=1}^{m}
    \varepsilon_{k,j,b}^{2}
    =
    m
    \left(
        \frac{1}{m}
        \sum_{j=1}^{m}
        \varepsilon_{k,j,b}^{2}
    \right).
\end{equation}
We define the component-specific relative distortion ratio of \(p_{k}\) as
\begin{equation}
    \label{eq:app_component_kappa_def}
    \kappa_{k,b}
    \coloneqq
    \frac{
        \frac{1}{m}
        \sum_{j=1}^{m}
        \varepsilon_{k,j,b}^{2}
    }{
        \frac{1}{m}
        \sum_{j=1}^{m}
        p_{k,j}^{2}
    } .
\end{equation}
Since \(p_{k}\) has unit norm, we have
\begin{equation}
    \label{eq:app_lowbit_unit_norm_expectation}
    \frac{1}{m}
    \sum_{j=1}^{m}
    p_{k,j}^{2}
    =
    \frac{1}{m}.
\end{equation}
Combining Eq.~\eqref{eq:app_component_kappa_def}, Eq.~\eqref{eq:app_lowbit_empirical_error_opt}, and Eq.~\eqref{eq:app_lowbit_unit_norm_expectation} gives
\begin{equation}
    \label{eq:app_lowbit_kappa_equals_error}
    \kappa_{k,b}
    =
    \frac{
        \frac{1}{m}
        \sum_{j=1}^{m}
        \varepsilon_{k,j,b}^{2}
    }{
        \frac{1}{m}
        \sum_{j=1}^{m}
        p_{k,j}^{2}
    }
    =
    m
    \left(
        \frac{1}{m}
        \sum_{j=1}^{m}
        \varepsilon_{k,j,b}^{2}
    \right)
    =
    \mathcal{E}_{k,b}.
\end{equation}
Since \(p_k\) has unit norm, its component-specific relative distortion ratio equals its total squared quantization error.

To calibrate a shared low-bit coefficient independently of model weights, let \(z\sim\mathcal{N}(0,I_{d_{\mathrm{ref}}})\) and \(u=z/\|z\|_2\), where \(d_{\mathrm{ref}}\) denotes the fixed dimension of the random vectors used to initialize \(\kappa_b\). We define the shared low-bit coefficient as the expected MSE-optimal distortion of these random unit vectors:
\begin{equation}
    \label{eq:app_kappa_def}
    \kappa_b
    \coloneqq
    \mathbb{E}_{z}
    \left[
        \min_{s>0}
        \left\|u-Q_b(u;s)\right\|_2^2
    \right].
\end{equation}
Repeated independent draws of \(u\), each with its own MSE-optimal scale, give a Monte Carlo estimate of \(\kappa_b\). Since \(u\) has unit norm, each squared error in Eq.~\eqref{eq:app_kappa_def} is also its relative distortion. Averaging these errors provides a model-independent starting point for the shared low-bit coefficient. Under the empirical coefficient-sharing approximation, we use the same \(\kappa_b\) as the low-bit surrogate for every spectral component across models, layers, experts, and projection types:
\begin{equation}
    \label{eq:app_low_bit_distortion}
    \mathcal{E}_{k,b}
    \approx
    \kappa_b,
    \qquad
    b\in\{2,3,4\}.
\end{equation}

The random-vector calibration of \(\kappa_b\) is detailed in App.~\ref{sec:ilp_coeff_analysis}; layer-wise and model-wide distortion diagnostics are reported in App.~\ref{app:distortion_dispersion}.

\paragraph{Bit-width-dependent distortion for 1-bit sign quantization.}
For \(b=1\), we adopt the symmetric 1-bit sign quantizer
\begin{equation}
    \label{eq:app_one_bit_quantization}
    Q_{1}(p_{k};s_{k,1})
    \coloneqq
    s_{k,1}\operatorname{sign}(p_{k}),
    \qquad
    s_{k,1}
    \coloneqq
    \frac{1}{m}
    \sum_{j=1}^{m}
    |p_{k,j}|.
\end{equation}
Define the normalized sign direction
\begin{equation}
    \label{eq:app_one_bit_sign_direction}
    \boldsymbol{r}_{k,1}
    \coloneqq
    \frac{\operatorname{sign}(p_{k})}{\sqrt{m}},
    \qquad
    \cos\theta_{k}
    \coloneqq
    p_{k}^{\top}\boldsymbol{r}_{k,1},
\end{equation}
where \(\operatorname{sign}(\cdot)\) is applied elementwise with \(\operatorname{sign}(0)=1\) so that \(\boldsymbol{r}_{k,1}\in\{\pm 1/\sqrt{m}\}^{m}\) and \(\|\boldsymbol{r}_{k,1}\|_2=1\).

The quantized spectral direction at \(b=1\) in Eq.~\eqref{eq:app_one_bit_quantization} can be rewritten using the normalized sign direction as
\begin{equation}
    \label{eq:app_one_bit_rewrite}
    Q_{1}(p_{k};s_{k,1})
    =
    s_{k,1}\operatorname{sign}(p_{k})
    =
    s_{k,1}\sqrt{m}\,\boldsymbol{r}_{k,1}.
\end{equation}
By the definition of \(\boldsymbol{r}_{k,1}\), its alignment with \(p_{k}\) is
\begin{equation}
    \label{eq:app_one_bit_cosine}
    \cos\theta_{k}
    =
    p_{k}^{\top}\boldsymbol{r}_{k,1}
    =
    \frac{1}{\sqrt{m}}
    \sum_{j=1}^{m}
    p_{k,j}\operatorname{sign}(p_{k,j})
    =
    \frac{1}{\sqrt{m}}
    \sum_{j=1}^{m}
    |p_{k,j}|
    =
    s_{k,1}\sqrt{m}.
\end{equation}
Therefore, the 1-bit reconstruction is equivalently
\begin{equation}
    \label{eq:app_one_bit_projection_form}
    Q_{1}(p_{k};s_{k,1})
    =
    \cos\theta_{k}\boldsymbol{r}_{k,1}.
\end{equation}
Thus the 1-bit quantization error vector is
\begin{equation}
    \label{eq:app_one_bit_error_vector}
    \boldsymbol{\varepsilon}_{k,1}
    \coloneqq
    p_{k}
    -
    Q_{1}(p_{k};s_{k,1})
    =
    p_{k}
    -
    \cos\theta_{k}\boldsymbol{r}_{k,1}.
\end{equation}
Since \(\|p_{k}\|_2=\|\boldsymbol{r}_{k,1}\|_2=1\), the bit-width-dependent distortion at \(b=1\) is
\begin{equation}
    \label{eq:app_one_bit_distortion}
    \begin{aligned}
    \mathcal{E}_{k,1}
    &=
    \left\|\boldsymbol{\varepsilon}_{k,1}\right\|_2^2
    =
    \left\|
        p_{k}
        -
        \cos\theta_{k}\boldsymbol{r}_{k,1}
    \right\|_2^2 \\
    &=
    \|p_{k}\|_2^2
    +
    \cos^2\theta_{k}
    \|\boldsymbol{r}_{k,1}\|_2^2
    -
    2\cos\theta_{k}
    p_{k}^{\top}\boldsymbol{r}_{k,1} \\
    &=
    1
    +
    \cos^2\theta_{k}
    -
    2\cos^2\theta_{k} \\
    &=
    1-\cos^2\theta_{k}
    =
    \sin^2\theta_{k}.
    \end{aligned}
\end{equation}

\paragraph{Bit-width-dependent distortion for zero-bit eviction.}
For \(b=0\), the expert-specific spectral direction is evicted:
\begin{equation}
    \label{eq:app_zero_bit_vector}
    Q_{0}(p_{k})
    \coloneqq
    \boldsymbol{0}.
\end{equation}
The resulting zero-bit error vector is
\begin{equation}
    \label{eq:app_zero_bit_error_vector}
    \boldsymbol{\varepsilon}_{k,0}
    \coloneqq
    p_{k}
    -
    Q_{0}(p_{k})
    =
    p_{k}.
\end{equation}
Taking the squared \(\ell_2\) norm gives the bit-width-dependent distortion at \(b=0\):
\begin{equation}
    \label{eq:app_zero_bit_distortion}
    \mathcal{E}_{k,0}
    =
    \left\|
        \boldsymbol{\varepsilon}_{k,0}
    \right\|_2^2
    =
    \left\|
        p_{k}
    \right\|_2^2
    =
    1,
\end{equation}
where the last equality follows from the unit norm of the expert-specific spectral direction.

\paragraph{Piecewise bit-width-dependent distortion.}
Combining the high-bit approximation in Eq.~\eqref{eq:app_high_bit_distortion}, the low-bit surrogate in Eq.~\eqref{eq:app_low_bit_distortion}, the 1-bit and zero-bit distortions in Eq.~\eqref{eq:app_one_bit_distortion} and Eq.~\eqref{eq:app_zero_bit_distortion}, and the FP16 reference \(\mathcal{E}_{k,16}=0\) yields the bit-width-dependent distortion used by \textsc{BitsMoE}:
\begin{equation}
    \label{eq:piecewise_expected_eps}
    \mathcal{E}_{k,b}
    \approx
    \begin{cases}
        0,
            & b=16, \\[3pt]
        \eta_k\exp(-\lambda b),
            & b\in\{6,8\}, \\[3pt]
        \kappa_b,
            & b\in\{2,3,4\}, \\[3pt]
        \sin^2\theta_k,
            & b=1, \\[3pt]
        1,
            & b=0.
    \end{cases}
\end{equation}

\subsection{Global Bit Budget by Projection Type}
\label{app:equivalent_bit_budget}

Following the notation table (Tab.~\ref{tab:app_detailed_notation}), \(m=d_{\mathrm{out}}\) denotes the dimension of each expert-specific spectral direction, and \(n\) denotes the number of spectral components per expert and projection type. In the full-component setting used here, \(n=n^{(h)}=d_{\mathrm{in}}\) for all \(h\in\mathcal{H}\). We consider \(L\) MoE layers with the same expert count \(E\) and dimensions \(m,n\). Each routed expert has three projection matrices with \(mn\) entries each. The target equivalent bit-width \(\mathfrak{b}_{\mathrm{eq}}\) therefore gives the total storage budget
\begin{equation}
    \label{eq:app_global_storage_budget}
    M_{\mathrm{budget}}
    =
    3LE mn\mathfrak{b}_{\mathrm{eq}}
    \quad \mathrm{bits}.
\end{equation}

Elsewhere in the paper, we report this target as the bit-width; for example, a 2-bit configuration uses $\mathfrak{b}_{\mathrm{eq}}=2$ under this storage accounting.

Each layer has one shared spectral basis per projection type, quantized to 8-bit precision using GPTQ. With all components retained, each basis has shape $n\times n$. The total shared-basis storage across all layers and projection types is therefore
\begin{equation}
    \label{eq:app_layer_shared_storage}
    M_{\mathrm{share}}
    =
    8\cdot 3L n^2
    \quad \mathrm{bits}.
\end{equation}

Subtracting the shared-basis storage in Eq.~\eqref{eq:app_layer_shared_storage} from the total storage budget in Eq.~\eqref{eq:app_global_storage_budget} gives the global bit budget for the expert-specific spectral directions:
\begin{equation}
    \label{eq:app_layer_budget}
    B
    =
    M_{\mathrm{budget}}-M_{\mathrm{share}}
    =
    3LE mn\mathfrak{b}_{\mathrm{eq}}
    -
    8\cdot 3L n^2
    \quad \mathrm{bits}.
\end{equation}
We split this global budget equally across the three projection types, assigning one third to each:
\begin{equation}
    \label{eq:app_projection_bit_budget}
    B_{\mathrm{bit}}^{(h)}
    =
    \frac{B}{3}
    =
    LE mn\mathfrak{b}_{\mathrm{eq}}
    -
    8L n^2,
    \qquad
    h\in\mathcal{H}.
\end{equation}

For each projection type, bits are allocated jointly across all layers under this global budget. For each layer \(\ell\) and projection type \(h\), the shared-basis decomposition produces \(n\) expert-specific spectral directions \(\{p_{\ell,e,k}^{(h)}\}_{k=1}^{n}\) for each expert, with \(p_{\ell,e,k}^{(h)}\in\mathbb{R}^{m}\). Assigning bit-width \(b\) to \(p_{\ell,e,k}^{(h)}\) consumes \(mb\) bits. We therefore normalize the global bit budget for projection type \(h\) by \(m\) and obtain
\begin{equation}
    \label{eq:app_projection_normalized_budget}
    B_{\mathrm{alloc}}^{(h)}(\mathfrak{b}_{\mathrm{eq}})
    \coloneqq
    \left\lfloor
        \frac{B_{\mathrm{bit}}^{(h)}}{m}
    \right\rfloor
    =
    \left\lfloor
        LE n\mathfrak{b}_{\mathrm{eq}}
        -
        \frac{8L n^2}{m}
    \right\rfloor.
\end{equation}
The resulting \(B_{\mathrm{alloc}}^{(h)}(\mathfrak{b}_{\mathrm{eq}})\) is the global normalized bit budget for projection type \(h\), corresponding to \(B_{\mathrm{alloc}}(\mathfrak{b}_{\mathrm{eq}})\) in Eq.~\eqref{eq:component_bit_allocation}. For each type, we minimize the global FQCM loss in Eq.~\eqref{eq:fqcm_loss} through a single ILP with the budget constraint
\begin{equation}
    \label{eq:app_projection_ilp_budget}
    \sum_{\ell=1}^{L}
    \sum_{e=1}^{E}
    \sum_{k=1}^{n}
    \sum_{b\in\mathcal{B}}
    b\,y_{\ell,e,k,b}^{(h)}
    \le
    B_{\mathrm{alloc}}^{(h)}(\mathfrak{b}_{\mathrm{eq}}),
\end{equation}
where \(y_{\ell,e,k,b}^{(h)}\in\{0,1\}\) indicates whether expert-specific spectral direction \(p_{\ell,e,k}^{(h)}\) is assigned bit-width \(b\), with
\begin{equation}
    \label{eq:app_projection_ilp_assignment}
    \sum_{b\in\mathcal{B}}
    y_{\ell,e,k,b}^{(h)}
    =
    1,
    \qquad
    \forall \ell\in[L],\ e\in[E],\ h\in\mathcal{H},\ k\in[n].
\end{equation}

\subsection{Component-wise Loss and ILP Formulation}
\label{app:component_wise_ilp}

The component-wise quantization cost in Eq.~\eqref{eq:app_unified_projection_loss} combines three factors. Restoring the layer index \(\ell\), we define them as follows.

\textbf{(1).~Intrinsic spectral importance.}
The squared component magnitude \(\bigl(\boldsymbol{\alpha}_{\ell,e,k}^{(h)}\bigr)^2=\|\widetilde{p}_{\ell,e,k}^{(h)}\|_2^2\) is the intrinsic spectral importance, measuring the calibration-agnostic strength of the SSD spectral component.

\textbf{(2).~Activation-dependent importance.}
The activation-dependent importance \(\boldsymbol{\beta}_{\ell,e,k}^{(h)}\) measures the component's sensitivity to calibration activations under the routing-weighted activation metric. For allocation, we use the normalized diagonal approximation described in App.~\ref{app:activation_weighted_loss}.

\textbf{(3).~Bit-width-dependent distortion.}
The bit-width-dependent distortion \(\mathcal{E}_{\ell,e,k,b}^{(h)}\coloneqq\mathbb{E}_{\epsilon}\|\boldsymbol{\varepsilon}_{\ell,e,k,b}^{(h)}\|_2^2\) measures the expected squared quantization error of the expert-specific spectral direction, where \(\boldsymbol{\varepsilon}_{\ell,e,k,b}^{(h)}=p_{\ell,e,k}^{(h)}-Q_b(p_{\ell,e,k}^{(h)})\). Its values for the candidate bit-widths are modeled by Eq.~\eqref{eq:piecewise_expected_eps}.

Combining these factors, the component-wise quantization cost of assigning bit-width \(b\) to \(p_{\ell,e,k}^{(h)}\) is approximated by
\begin{equation}
    \label{eq:app_component_loss_final}
    \mathcal{L}_{\ell,e,k}^{(h)}(b)
    \approx
    \underbrace{\bigl(\boldsymbol{\alpha}_{\ell,e,k}^{(h)}\bigr)^2}_{\substack{\text{Intrinsic spectral}\\\text{importance}}}
    \;\cdot\;
    \underbrace{\boldsymbol{\beta}_{\ell,e,k}^{(h)}}_{\substack{\text{Activation-dependent}\\\text{importance}}}
    \;\cdot\;
    \underbrace{\mathcal{E}_{\ell,e,k,b}^{(h)}}_{\substack{\text{Bit-width-dependent}\\\text{distortion}}}.
\end{equation}

For each projection type \(h\in\mathcal{H}\), let \(y_{\ell,e,k,b}^{(h)}\in\{0,1\}\) indicate whether spectral component \((\ell,e,h,k)\) is assigned bit-width \(b\in\mathcal{B}\), and let \(y\) collect these assignment variables. As in Sec.~\ref{sec:ilp_model}, we define the global FQCM loss by summing the selected component-wise quantization costs over all \(L\) MoE layers:
\begin{equation}
    \label{eq:app_global_fqcm_loss}
    \mathcal{L}_{\mathrm{FQCM}}^{(h)}(y)
    \coloneqq
    \sum_{\ell=1}^{L}\sum_{e=1}^{E}\sum_{k=1}^{n}\sum_{b\in\mathcal{B}}
    \mathcal{L}_{\ell,e,k}^{(h)}(b)y_{\ell,e,k,b}^{(h)}.
\end{equation}

Given the global normalized bit budget \(B_{\mathrm{alloc}}^{(h)}(\mathfrak{b}_{\mathrm{eq}})\) in Eq.~\eqref{eq:app_projection_normalized_budget}, we minimize this global FQCM loss through a single ILP, following Eq.~\eqref{eq:component_bit_allocation}:
\begin{equation}
    \label{eq:ilp_piecewise_eps}
    \begin{aligned}
        \textsc{Minimize}\quad &\mathcal{L}_{\mathrm{FQCM}}^{(h)}(y) \\
        \textsc{s.t.}\quad
        &\sum_{b\in\mathcal{B}}y_{\ell,e,k,b}^{(h)}=1,
        \qquad \forall\,\ell\in[L],\ e\in[E],\ k\in[n], \\
        &\sum_{\ell=1}^{L}\sum_{e=1}^{E}\sum_{k=1}^{n}\sum_{b\in\mathcal{B}}
        b\,y_{\ell,e,k,b}^{(h)}\le B_{\mathrm{alloc}}^{(h)}(\mathfrak{b}_{\mathrm{eq}}), \\
        &y_{\ell,e,k,b}^{(h)}\in\{0,1\},
        \qquad \forall\,\ell\in[L],\ e\in[E],\ k\in[n],\ b\in\mathcal{B}.
    \end{aligned}
\end{equation}
We solve this global ILP independently for each projection type, allocating bits jointly across all MoE layers within its budget.

\section{Empirical Validation of Shared Distortion Coefficients}
\label{app:cost_model_validation}

We examine the empirical basis for the shared low-bit distortion coefficients $\kappa_b$ defined in Eq.~\eqref{eq:app_kappa_def}. We describe their random-vector calibration and assess the dispersion of measured distortions in expert-specific spectral directions, both within layers and across all MoE layers for each projection type.

\subsection{Calibration of Low-bit Distortion Coefficients}
\label{sec:ilp_coeff_analysis}

Following the empirical bit-width-dependent distortion formulation at low bit-widths in App.~\ref{app:piecewise_reconstruction_error}, we estimate $\kappa_b$ for $b\in\{2,3,4\}$ using independent Gaussian vectors $z_i\sim\mathcal{N}(0,I_{d_{\mathrm{ref}}})$ normalized as $u_i=z_i/\|z_i\|_2$. This normalization matches the unit norm of the expert-specific spectral directions.

For each vector and bit-width, we use the quantizer in Eq.~\eqref{eq:app_lowbit_quantizer}, with integer range $[-2^{b-1},2^{b-1}-1]$ and one scale for the complete vector. The scale minimizes the empirical squared error:
\begin{equation}
    \label{eq:kappa_random_vector_error}
    s_{i,b}^{\ast}\in\arg\min_{s>0}\|u_i-Q_b(u_i;s)\|_2^2,
    \qquad
    D_{i,b}\coloneqq\|u_i-Q_b(u_i;s_{i,b}^{\ast})\|_2^2.
\end{equation}
We repeat this sampling-and-quantization procedure and average the resulting squared errors to estimate the reference coefficient in Eq.~\eqref{eq:app_kappa_def}:
\begin{equation}
    \label{eq:kappa_monte_carlo_estimate}
    \widehat{\kappa}_b=\frac{1}{N}\sum_{i=1}^{N}D_{i,b},
\end{equation}
where $N$ denotes the number of random draws. Since $\|u_i\|_2=1$, $D_{i,b}$ is also the relative distortion and is not divided by the vector dimension. This calibration provides a model-independent starting point for $\kappa_b$. Tab.~\ref{tab:kappa_values} lists the shared low-bit distortion coefficients used in the model experiments; applying each coefficient across models, layers, experts, projection types, and spectral components follows the empirical coefficient-sharing approximation in Eq.~\eqref{eq:app_low_bit_distortion}.

\begin{table}[th]
    \centering
    \tablestyle
    \caption{Shared low-bit distortion coefficients $\kappa_b$ used in the model experiments.}
    \label{tab:kappa_values}
    \begin{tabular}{@{}cccc@{}}
        \toprule
        Bit-width $b$ & 4-bit & 3-bit & 2-bit \\
        \midrule
        $\kappa_b$ & 0.01184786 & 0.04067890 & 0.14949200 \\
        \bottomrule
    \end{tabular}
\end{table}

\subsection{Analysis of Component-wise Quantization Costs}
\label{app:ilp_coefficient_consistency}

\begin{figure}[th]
    \centering

    \begin{subfigure}[t]{\linewidth}
        \centering
        \includegraphics[width=\linewidth]{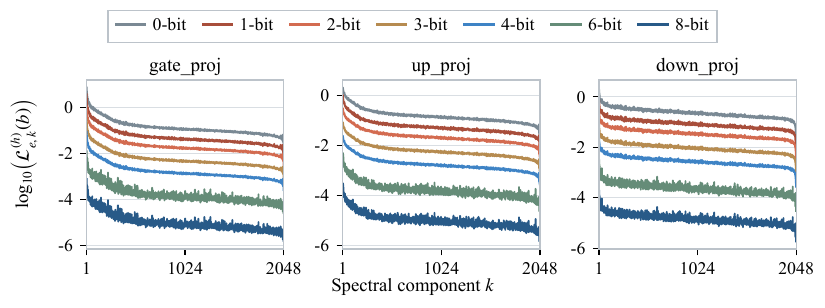}
        \caption{Layer 1, Expert 1.}
        \label{fig:deepseekv2_ch_l1_e0}
    \end{subfigure}

    \vspace{4pt}

    \begin{subfigure}[t]{\linewidth}
        \centering
        \includegraphics[width=\linewidth]{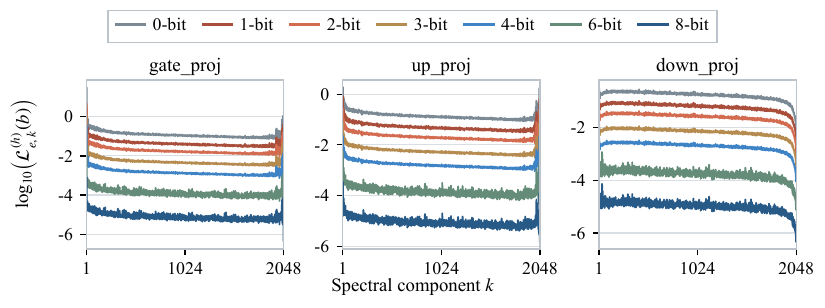}
        \caption{Layer 23, Expert 64.}
        \label{fig:deepseekv2_ch_l23_e63}
    \end{subfigure}

    \caption{
    Component-wise quantization costs $\mathcal{L}_{e,k}^{(h)}(b)$ across spectral components and bit-widths for selected layers and experts in DeepSeek-V2-Lite.
    }
    \label{fig:deepseekv2_ch_coeffs}
\end{figure}

Fig.~\ref{fig:deepseekv2_ch_coeffs} shows the component-wise quantization costs $\mathcal{L}_{e,k}^{(h)}(b)$ defined in Eq.~\eqref{eq:app_unified_projection_loss} for selected layers and experts in DeepSeek-V2-Lite. These costs combine intrinsic spectral importance $\bigl(\boldsymbol{\alpha}_{e,k}^{(h)}\bigr)^2$, activation-dependent importance $\boldsymbol{\beta}_{e,k}^{(h)}$, and bit-width-dependent distortion $\mathcal{E}_{e,k,b}^{(h)}$, illustrating how the modeled quantization penalty varies with both the spectral component and the assigned precision.

Fig.~\ref{fig:deepseekv2_ch_l23_e63} also shows that the component-wise quantization costs for $\mathrm{up\_proj}$ do not decrease monotonically with the spectral component index $k$: some components near the tail have substantially higher costs than those in the middle of the spectrum. The underlying statistics show increased activation-dependent importance $\boldsymbol{\beta}_{e,k}^{(h)}$ at these tail components, offsetting their smaller intrinsic spectral importance $\bigl(\boldsymbol{\alpha}_{e,k}^{(h)}\bigr)^2$. For a given spectral component, higher bit-widths generally yield lower component-wise quantization costs, consistent with the reduced bit-width-dependent distortion.

\subsection{Dispersion of Bit-width-dependent Distortion}
\label{app:distortion_dispersion}

The following layer-wise statistics characterize measured bit-width-dependent distortions of model spectral directions separately from the random-vector calibration in App.~\ref{sec:ilp_coeff_analysis}. For each layer \(\ell\) and bit-width \(b\), we average the bit-width-dependent distortions across all projection types, experts, and spectral components:
\begin{equation}
    \label{eq:kappa_empirical_estimate}
    \overline{\mathcal{E}}_{\ell,b}
    =
    \frac{1}{E\sum_{h\in\mathcal{H}} n^{(h)}}
    \sum_{h\in\mathcal{H}}
    \sum_{e=1}^{E}
    \sum_{k=1}^{n^{(h)}}
    \left\|
    p_{\ell,e,k}^{(h)}
    -
    Q_b(p_{\ell,e,k}^{(h)})
    \right\|_2^2,
\end{equation}
where \(E\) is the number of routed experts, \(\mathcal{H}\) is the set of projection types, and \(n^{(h)}\) is the number of spectral components for projection type \(h\). We use the coefficient of variation (CV) to quantify the relative dispersion of bit-width-dependent distortions within each layer:
\begin{equation}
    \label{eq:cv_kappa}
    \mathrm{CV}_{\ell,b}
    =
    \frac{s_{\ell,b}}
    {\left|\overline{\mathcal{E}}_{\ell,b}\right|}
    \times 100\%,
\end{equation}
where \(s_{\ell,b}\) denotes the sample standard deviation of the bit-width-dependent distortions over all projection types, experts, and spectral components in layer \(\ell\). Fig.~\ref{fig:kappa_cv_overview} reports the layer-wise CV under 2-bit, 3-bit, and 4-bit quantization on DeepSeek-V2-Lite. The CV values are almost all below \(15\%\), indicating low relative dispersion in the measured bit-width-dependent distortions.

\begin{figure}[th]
    \centering
    \includegraphics[width=0.8\linewidth]{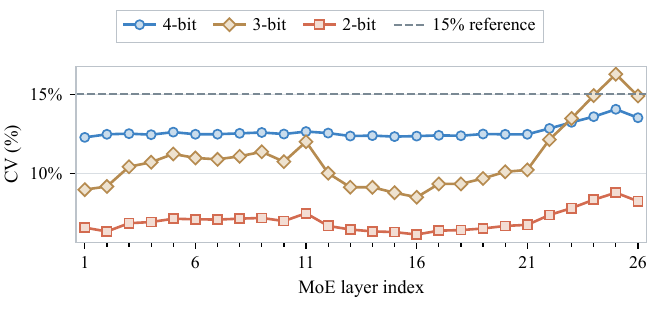}
    \caption{
    Layer-wise CV of measured bit-width-dependent distortions on DeepSeek-V2-Lite.
    }
    \label{fig:kappa_cv_overview}
\end{figure}

To match the scope of the global ILP, which allocates bits across all MoE layers independently for each projection type, we additionally pool the measured bit-width-dependent distortions across all MoE layers of DeepSeek-V2-Lite. For each projection type and bit-width, we compute the CV as the pooled sample standard deviation divided by the pooled mean, expressed as a percentage. This statistic includes both within-layer dispersion and differences between layer means.

\begin{table}[th]
    \centering
    \tablestyle
    \caption{CV (\%) of bit-width-dependent distortions pooled across all MoE layers, experts, and spectral components for each projection type in DeepSeek-V2-Lite.}
    \label{tab:model_wide_distortion_cv}
    \begin{tabular}{@{}cccc@{}}
        \toprule
        Projection type & 2-bit & 3-bit & 4-bit \\
        \midrule
        $\mathrm{gate\_proj}$ & 8.3215 & 11.2069 & 12.5066 \\
        $\mathrm{up\_proj}$ & 5.9475 & 10.0072 & 12.5263 \\
        $\mathrm{down\_proj}$ & 6.6764 & 12.1075 & 12.9585 \\
        \bottomrule
    \end{tabular}
\end{table}

Tab.~\ref{tab:model_wide_distortion_cv} shows that the pooled CV values remain below $13\%$ for all three projection types and bit-widths. The values range from $5.95\%$ to $8.32\%$ at 2-bit, from $10.01\%$ to $12.11\%$ at 3-bit, and from $12.51\%$ to $12.96\%$ at 4-bit. These results indicate low relative dispersion in measured bit-width-dependent distortions across the full set of components considered by each global ILP.

\section{Additional Accuracy Results and Sensitivity Analysis}
\label{app:detailed_experiments}

This section presents additional experimental results: task-level accuracy under ultra-low-bit quantization (App.~\ref{app:ultralow_accuracy}), task-level accuracy for the ablation studies (App.~\ref{app:ablation_supplement}), and sensitivity to calibration data composition (App.~\ref{app:calibration_dataset_composition}).

\subsection{Accuracy under Ultra-Low-Bit Quantization}
\label{app:ultralow_accuracy}

We evaluate \textsc{BitsMoE} on Qwen3-30B-A3B and Qwen3-Next-80B-A3B at 1.4--2.0-bit precision. Tab.~\ref{tab:qwen_ultralow_full} reports task-level and average accuracy, with FP16 references and 2-bit HQQ/GPTQ baselines. For Qwen3-Next-80B-A3B, we additionally compare against 2-bit MiLo and MxMoE.

\textsc{BitsMoE} remains competitive with 2-bit baselines even at 1.4-bit. On Qwen3-30B-A3B, it achieves an average accuracy of 45.07\%, exceeding 2-bit HQQ's 43.18\% by 1.89 percentage points. On Qwen3-Next-80B-A3B, 1.4-bit \textsc{BitsMoE} achieves 64.14\% average accuracy, surpassing 2-bit MxMoE's 51.14\% by 13.00 percentage points. Although both methods use ILP-based mixed-precision bit allocation, \textsc{BitsMoE} achieves higher average accuracy in this comparison with a lower target bit-width.

The GSM8K results further demonstrate the preservation of mathematical reasoning performance below 2-bit precision. At 1.6-bit, \textsc{BitsMoE} retains 54.06\% GSM8K accuracy on Qwen3-30B-A3B, exceeding 2-bit HQQ's 22.67\% by 31.39 percentage points. On Qwen3-Next-80B-A3B, it achieves 66.79\% at 1.6-bit, slightly above 2-bit HQQ's 66.49\% by 0.30 percentage points. These results further support the effectiveness of \textsc{BitsMoE} for generation tasks requiring multi-step mathematical reasoning under constrained model-storage budgets.

\begin{table}[th]
    \centering
    \caption{Task-level accuracy of \textsc{BitsMoE} on Qwen3-30B-A3B at 1.4--2.0-bit and Qwen3-Next-80B-A3B at 1.4--2.0-bit.}
    \label{tab:qwen_ultralow_full}
    \tablestyle
    \begin{tabular}{@{}ccc*{8}{c}@{}}
        \toprule
        \multirow{2}{*}{\textbf{Model}} & \multirow{2}{*}{\textbf{Method}} & \multirow{2}{*}{\textbf{Bits}} & \multicolumn{8}{c}{\textbf{Accuracy} (\%) $\uparrow$} \\
        \cmidrule(lr){4-11}
        & & & \textbf{HS} & \textbf{AE} & \textbf{MQ} & \textbf{MM} & \textbf{OB} & \textbf{WG} & \textbf{GS} & \textbf{Avg.} \\
        \midrule
        \multirow{10}{*}{\makecell{\textsc{Qwen3-}\\\textsc{30B-A3B}}}
        & FP16 & 16 & 81.35 & 79.80 & 60.03 & 78.77 & 45.00 & 72.85 & 83.47 & 71.61 \\
        \cmidrule(lr){2-11}
        & HQQ & 2 & 63.05 & 37.33 & 33.17 & 48.82 & 36.40 & 60.85 & 22.67 & 43.18 \\
        & GPTQ & 2 & 70.16 & 59.97 & 24.89 & 39.17 & 39.40 & 60.62 & 4.32 & 42.65 \\
        \cmidrule(lr){2-11}
        & \multirow{7}{*}{\textbf{\textsc{BitsMoE}}} & 2.0 & \textbf{70.43} & \textbf{80.18} & \textbf{52.86} & \textbf{64.36} & \textbf{41.00} & \textbf{70.17} & \textbf{71.04} & \textbf{64.29} \\
        & & 1.9 & 68.57 & 77.99 & 51.22 & 61.78 & 38.60 & 69.53 & 66.87 & 62.08 \\
        & & 1.8 & 66.21 & 76.14 & 48.41 & 60.40 & 37.00 & 67.40 & 64.75 & 60.04 \\
        & & 1.7 & 63.00 & 74.03 & 46.43 & 54.94 & 36.20 & 65.43 & 58.15 & 56.88 \\
        & & 1.6 & 60.57 & 73.70 & 43.45 & 48.18 & 35.00 & 65.04 & 54.06 & 54.29 \\
        & & 1.5 & 57.16 & 72.47 & 39.73 & 45.66 & 34.00 & 62.51 & 43.29 & 50.69 \\
        & & 1.4 & 51.08 & 68.18 & 36.68 & 39.63 & 31.80 & 57.54 & 30.55 & 45.07 \\
        \midrule
        \multirow{12}{*}{\makecell{\textsc{Qwen3-Next-}\\\textsc{80B-A3B}}}
        & FP16 & 16 & 82.72 & 84.64 & 63.85 & 84.53 & 44.20 & 76.80 & 77.18 & 73.42 \\
        \cmidrule(lr){2-11}
        & HQQ & 2 & \textbf{78.73} & 83.50 & 50.95 & 79.68 & 43.80 & 70.40 & 66.49 & 67.65 \\
        & GPTQ & 2 & 70.24 & \textbf{84.30} & 27.47 & 54.63 & 38.60 & 65.59 & 18.35 & 51.31 \\
        & MiLo & 2 & 78.69 & 79.17 & 49.75 & 79.76 & 43.80 & 71.74 & 71.49 & 67.77 \\
        & MxMoE & 2 & 80.68 & 67.26 & 24.96 & 65.20 & 44.40 & 74.51 & 0.99 & 51.14 \\
        \cmidrule(lr){2-11}
        & \multirow[c]{7}{*}{\textbf{\textsc{BitsMoE}}} & 2.0 & 78.05 & 83.92 & \textbf{61.71} & \textbf{79.95} & \textbf{44.60} & \textbf{74.90} & \textbf{71.80} & \textbf{70.70} \\
        &  & 1.9 & 77.45 & 83.42 & 61.74 & 79.11 & 44.40 & 75.22 & 70.28 & 70.23 \\
        &  & 1.8 & 76.37 & 83.04 & 59.80 & 77.93 & 43.80 & 74.59 & 69.07 & 69.23 \\
        &  & 1.7 & 75.07 & 84.09 & 58.46 & 77.13 & 44.00 & 74.74 & 69.60 & 69.01 \\
        &  & 1.6 & 73.39 & 83.63 & 56.52 & 75.97 & 42.60 & 74.19 & 66.79 & 67.58 \\
        &  & 1.5 & 71.47 & 83.04 & 55.64 & 73.69 & 42.60 & 73.24 & 65.05 & 66.39 \\
        &  & 1.4 & 68.66 & 82.24 & 53.50 & 71.39 & 39.40 & 71.35 & 62.47 & 64.14 \\
        \bottomrule
    \end{tabular}
\end{table}

\subsection{Supplementary Ablation Results}
\label{app:ablation_supplement}

Tab.~\ref{tab:ablation_combined_full} provides task-level results for the ablations on Qwen3-30B-A3B at 2-bit, complementing the average accuracies in Tab.~\ref{tab:ablation_combined}. Uniform assigns equal budgets to all experts and quantizes spectral components to 2-bit in descending spectral energy order until each expert's budget is exhausted, assigning the remaining components to 0-bit. Consistent with Sec.~\ref{sec:ablation}, SSD with uniform allocation yields limited gains over independent SVD, whereas incorporating activation-dependent importance and routing weights substantially improves accuracy. Combining intrinsic spectral importance, activation-dependent importance, and routing weights achieves the highest average accuracy, supporting their complementary roles in component-wise mixed-precision bit allocation.

The task-level results also clarify the trade-off between shared-basis precision and the budget available for expert-specific spectral directions. On HellaSwag, a 4-bit shared basis achieves 79.71\% accuracy, exceeding the 70.43\% achieved with an 8-bit shared basis. However, the 8-bit shared basis performs better on the other six tasks and improves average accuracy from 62.78\% to 64.29\%. In particular, on GSM8K, which requires generating solutions to multi-step mathematical problems, it improves accuracy from 65.28\% to 71.04\%, a gain of 5.76 percentage points. At the same target equivalent bit-width, reducing shared-basis precision to 4-bit frees more of the storage budget for expert-specific spectral directions, but this additional budget does not translate into higher average accuracy. This pattern is consistent with the concern that quantization errors in the shared transformations can offset the benefits of allocating more bits to expert-specific spectral directions. These results support an 8-bit shared basis as a more favorable balance between shared-basis precision and expert-specific storage under the evaluated setting.

\begin{table}[th]
    \centering
    \caption{Task-level accuracy for SSD, FQCM, and shared-basis precision ablations on Qwen3-30B-A3B at 2-bit. The Bits column specifies shared-basis storage precision.}
    \label{tab:ablation_combined_full}
    \tablestyle
    \setlength{\tabcolsep}{4.5pt}
    \begin{tabular}{@{}cccccc*{8}{c}@{}}
        \toprule
        \multirow{2}{*}{\textbf{Basis}} & \multirow{2}{*}{\textbf{Bits}} & \multirow{2}{*}{\textbf{Alloc.}} & \multirow{2}{*}{$\boldsymbol{\alpha}^2$} & \multirow{2}{*}{$\boldsymbol{\beta}$} & \multirow{2}{*}{$G_e$} & \multicolumn{8}{c}{\textbf{Accuracy} (\%) $\uparrow$} \\
        \cmidrule(lr){7-14}
        & & & & & & \textbf{HS} & \textbf{AE} & \textbf{MQ} & \textbf{MM} & \textbf{OB} & \textbf{WG} & \textbf{GS} & \textbf{Avg.} \\
        \midrule
        Non-Basis & -- & Uniform & -- & -- & -- & 26.50 & 24.87 & 19.87 & 23.24 & 28.40 & 50.51 & 0.61 & 24.86 \\
        \midrule
        \multirow{9}{*}[-5pt]{\textsc{SSD}} & \multirow{6}{*}[-2.5pt]{8-bit} & Uniform & -- & -- & -- & 29.31 & 28.45 & 19.53 & 23.03 & 26.40 & 50.20 & 0.00 & 25.27 \\
        \cmidrule(l){3-14}
        & & \multirow{5}{*}{FQCM} & \xmark & \xmark & -- & 30.34 & 30.39 & 21.11 & 23.05 & 26.60 & 50.67 & 1.06 & 26.17 \\
        & & & \cmark & \xmark & -- & 32.13 & 32.24 & 19.77 & 23.56 & 26.80 & 50.28 & 0.61 & 26.48 \\
        & & & \xmark & \cmark & \xmark & 54.29 & 71.38 & 33.70 & 55.10 & 33.20 & 62.35 & 20.39 & 47.20 \\
        & & & \xmark & \cmark & \cmark & 65.32 & 77.10 & 44.12 & 60.55 & 37.40 & 66.61 & 54.36 & 57.92 \\
        & & & \cmark & \cmark & \xmark & 67.23 & 76.85 & 46.37 & 61.92 & 38.20 & 67.48 & 65.13 & 60.45 \\
        \cmidrule(l){2-14}
        & FP16 & \multirow{3}{*}{FQCM} & \cmark & \cmark & \cmark & 76.47 & 67.04 & 49.28 & 61.27 & 38.20 & 67.80 & 66.34 & 60.91 \\
        & 4-bit & & \cmark & \cmark & \cmark & \textbf{79.71} & 69.81 & 52.66 & 63.78 & 39.40 & 68.82 & 65.28 & 62.78 \\
        & 8-bit & & \cmark & \cmark & \cmark & 70.43 & \textbf{80.18} & \textbf{52.86} & \textbf{64.36} & \textbf{41.00} & \textbf{70.17} & \textbf{71.04} & \textbf{64.29} \\
        \bottomrule
    \end{tabular}
\end{table}

\subsection{Sensitivity to Calibration Data Composition}
\label{app:calibration_dataset_composition}

Tab.~\ref{tab:calibration_mix_full} examines calibration-domain sensitivity on Qwen3-30B-A3B at 2-bit by varying the mixture of C4~\citep{raffel2020exploring} and MATH~\citep{hendrycksmath2021} while keeping the calibration set fixed at 1,024 samples.

C4-only calibration achieves the highest average accuracy of 64.29\%, while the C4--MATH mixtures remain close at 64.11--64.27\%. Calibration on C4 helps identify spectral components important for preserving general language modeling ability, while adding MATH can improve mathematical reasoning performance. The 50\% C4--50\% MATH mixture raises MathQA accuracy from 52.86\% to 53.90\% and GSM8K accuracy from 71.04\% to 72.78\% relative to C4-only calibration. However, using MATH alone further improves MathQA accuracy to 55.48\% but lowers GSM8K accuracy to 70.20\% and average accuracy to 63.01\%. This decline suggests that mathematical inputs alone may not adequately capture the importance of spectral components needed for general language modeling, which also supports solving mathematical word problems. These results support retaining general-domain data in calibration, with C4-only calibration providing the best average accuracy among the evaluated settings.

\begin{table}[th]
    \centering
    \caption{Task-level results for different C4--MATH calibration mixtures on Qwen3-30B-A3B under 2-bit quantization.}
    \label{tab:calibration_mix_full}
    \tablestyle
    \begin{tabular}{@{}cc*{8}{c}@{}}
        \toprule
        \multirow{2}{*}{\textbf{C4}} & \multirow{2}{*}{\textbf{MATH}}
        & \multicolumn{8}{c}{\textbf{Accuracy} (\%) $\uparrow$} \\
        \cmidrule(lr){3-10}
        & & \textbf{HS} & \textbf{AE} & \textbf{MQ} & \textbf{MM}
        & \textbf{OB} & \textbf{WG} & \textbf{GS} & \textbf{Avg.} \\
        \midrule
        100\% & 0\%
        & 70.43 & 80.18 & 52.86 & 64.36 & 41.00 & 70.17 & 71.04 & 64.29 \\
        80\% & 20\%
        & 70.83 & 79.55 & 54.00 & 64.89 & 39.00 & 71.03 & 70.58 & 64.27 \\
        50\% & 50\%
        & 70.39 & 79.29 & 53.90 & 63.71 & 39.80 & 69.85 & 72.78 & 64.25 \\
        20\% & 80\%
        & 70.77 & 78.75 & 54.97 & 63.62 & 39.20 & 69.77 & 71.72 & 64.11 \\
        0\% & 100\%
        & 68.02 & 78.11 & 55.48 & 62.94 & 37.20 & 69.14 & 70.20 & 63.01 \\
        \bottomrule
    \end{tabular}
\end{table}

\section{Inference Efficiency and Offline Quantization Costs}
\label{app:gpu_execution_efficiency}

This section presents efficiency measurements for \textsc{BitsMoE}: time to first token, decode throughput, and peak GPU memory (App.~\ref{app:absolute_efficiency}), and offline quantization time with a per-stage breakdown (App.~\ref{app:quantization_time_breakdown}).

\subsection{Inference Efficiency and GPU Memory Usage}
\label{app:absolute_efficiency}

We evaluate inference efficiency on two devices, a single RTX A6000 (48\,GB) and a single RTX 3090 (24\,GB), using the same benchmark parameters summarized in Tab.~\ref{tab:inference_benchmark_settings}. We compare GPTQ and \textsc{BitsMoE} at 2-bit using prompts from a pool of 512 samples drawn from the English C4 training split~\citep{raffel2020exploring}. Measurements are averaged over three successful runs to reduce the impact of run-to-run variability. Decode throughput is measured in generated tokens per GPU decode second, and memory denotes peak PyTorch-allocated GPU memory. GPTQ uses the Triton backend described in Sec.~\ref{sec:efficiency_analysis}.

\begin{table}[htbp]
    \centering
    \caption{Settings for the inference efficiency and GPU memory benchmarks on a single RTX A6000 (48\,GB) or RTX 3090 (24\,GB).}
    \label{tab:inference_benchmark_settings}
    \tablestyle
    \begin{tabular}{@{}cc@{}}
        \toprule
        \textbf{Setting} & \textbf{Value} \\
        \midrule
        \multirow{2}{*}{GPU} & 1 RTX A6000-48GB \\
        & 1 RTX 3090-24GB \\
        \midrule
        Prompt dataset & C4-en \\
        \texttt{max\_samples} & 512 \\
        \texttt{max\_length} & 256 \\
        \texttt{truncation} & \texttt{True} \\
        \texttt{padding} & \texttt{max\_length} \\
        \texttt{max\_new\_tokens} & 128 \\
        \texttt{do\_sample} & \texttt{False} \\
        \texttt{batch\_size} & 1, 2, 4, 8, 16, 32, 64, 128, 256 \\
        \texttt{batch\_size} (streaming) & 1 \\
        \bottomrule
    \end{tabular}
\end{table}

Fig.~\ref{fig:batch_memory} reports peak GPU memory across batch sizes on a single RTX A6000, complementing the decode-throughput results in Fig.~\ref{fig:decode_throughput}. \textsc{BitsMoE} and GPTQ have similar memory footprints on DeepSeek-V2-Lite and Qwen3-Next-80B-A3B. On Qwen3-30B-A3B, \textsc{BitsMoE} uses modestly more memory, with the gap increasing from 0.29\,GB at batch size 1 to 2.40\,GB at batch size 256. The additional storage supports packing and quantization metadata and intermediate buffers used during inference. In conjunction with the throughput gains in Fig.~\ref{fig:decode_throughput}, this modest overhead represents a practical tradeoff among ultra-low-bit precision, inference speed, and GPU memory usage. Both methods run out of memory at batch size 256 on DeepSeek-V2-Lite and Qwen3-Next-80B-A3B, as indicated by the shaded regions.

\begin{figure}[ht]
    \centering
    \includegraphics[width=\textwidth]{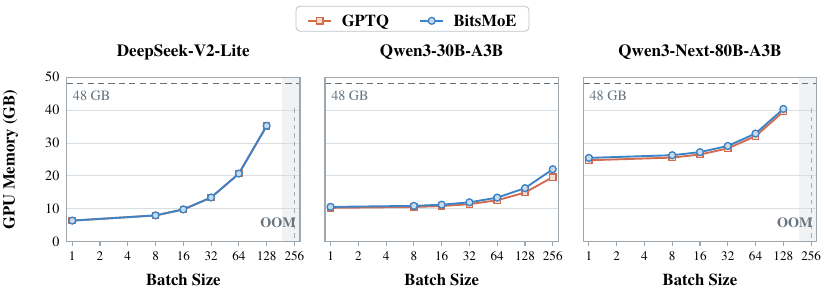}
    \caption{Peak GPU memory of GPTQ and \textsc{BitsMoE} across batch sizes.}
    \label{fig:batch_memory}
\end{figure}

We further evaluate streaming inference at batch size 1 on the RTX A6000, reporting time to first token (TTFT), decode throughput, and peak GPU memory in Tab.~\ref{tab:a6000_streaming_bs1_efficiency}. \textsc{BitsMoE} reduces TTFT and improves decode throughput across model scales, from DeepSeek-V2-Lite to Qwen3-Next-80B-A3B. For Qwen3-Next-80B-A3B, whose FP16 GPU memory footprint is 148.69\,GB (Fig.~\ref{fig:memory_footprint}), the 2-bit model runs on a single 48\,GB RTX A6000 with a peak allocation of 25.46\,GB. It achieves 5.13\,tok/s versus GPTQ's 2.44\,tok/s, a $2.10\times$ decode speedup, while reducing TTFT from 7.13\,s to 0.99\,s. These results demonstrate the value of ultra-low-bit deployment for serving large MoE models under single-GPU memory constraints.

\begin{table}[htbp]
    \centering
    \caption{Time to first token, decode throughput, and peak GPU memory at 2-bit for streaming inference at batch size 1 on a single RTX A6000 (48\,GB) or RTX 3090 (24\,GB).}
    \label{tab:streaming_bs1_efficiency}
    \begin{subtable}[t]{\linewidth}
        \centering
        \caption{On RTX A6000 (48\,GB).}
        \label{tab:a6000_streaming_bs1_efficiency}
        \tablestyle
        \setlength{\tabcolsep}{3pt}
    \begin{tabular}{@{}c*{8}{c}@{}}
        \toprule
        \multirow{2}{*}{\textbf{Model}}
        & \multicolumn{3}{c}{\textbf{TTFT (s)} $\downarrow$}
        & \multicolumn{3}{c}{\textbf{Decode throughput (tok/s)} $\uparrow$}
        & \multicolumn{2}{c}{\textbf{Memory (GB)} $\downarrow$} \\
        \cmidrule(lr){2-4}\cmidrule(lr){5-7}\cmidrule(lr){8-9}
        & GPTQ & \textsc{BitsMoE} & Speedup
        & GPTQ & \textsc{BitsMoE} & Speedup
        & GPTQ & \textsc{BitsMoE} \\
        \midrule
        \textsc{DeepSeek-V2-Lite} & 1.57 & 0.61 & \textbf{2.57$\times$} & 6.99 & 9.58 & \textbf{1.37$\times$} & 6.39 & 6.40 \\
        \textsc{Qwen3-30B-A3B} & 3.00 & 0.85 & \textbf{3.53$\times$} & 3.08 & 7.19 & \textbf{2.33$\times$} & 10.28 & 10.56 \\
        \textsc{Qwen3-Next-80B-A3B} & 7.13 & 0.99 & \textbf{7.20$\times$} & 2.44 & 5.13 & \textbf{2.10$\times$} & 24.73 & 25.46 \\
        \bottomrule
    \end{tabular}
    \end{subtable}

    \vspace{6pt}
    \begin{subtable}[t]{\linewidth}
        \centering
        \caption{On RTX 3090 (24\,GB).}
        \label{tab:rtx3090_streaming_bs1_efficiency}
        \tablestyle
        \setlength{\tabcolsep}{3pt}
    \begin{tabular}{@{}c*{8}{c}@{}}
        \toprule
        \multirow{2}{*}{\textbf{Model}}
        & \multicolumn{3}{c}{\textbf{TTFT (s)} $\downarrow$}
        & \multicolumn{3}{c}{\textbf{Decode throughput (tok/s)} $\uparrow$}
        & \multicolumn{2}{c}{\textbf{Memory (GB)} $\downarrow$} \\
        \cmidrule(lr){2-4}\cmidrule(lr){5-7}\cmidrule(lr){8-9}
        & GPTQ & \textsc{BitsMoE} & Speedup
        & GPTQ & \textsc{BitsMoE} & Speedup
        & GPTQ & \textsc{BitsMoE} \\
        \midrule
        \textsc{Qwen3-30B-A3B} & 2.32 & 0.84 & \textbf{2.76$\times$} & 3.97 & 7.52 & \textbf{1.89$\times$} & 10.28 & 10.56 \\
        \bottomrule
    \end{tabular}
    \end{subtable}
\end{table}

On a single RTX 3090 with 24\,GB of memory, \textsc{BitsMoE} also achieves lower TTFT and higher decode throughput than GPTQ on Qwen3-30B-A3B. Tab.~\ref{tab:rtx3090_inference_efficiency} reports inference performance across batch sizes 1--256, while Tab.~\ref{tab:rtx3090_streaming_bs1_efficiency} reports streaming inference at batch size 1. Decode speedups range from $1.96\times$ to $4.95\times$, computed from the reported means before rounding. At batch size 256, \textsc{BitsMoE} reaches 242.61\,tok/s versus GPTQ's 81.02\,tok/s, a $2.99\times$ speedup, with peak allocated memory of 22.01\,GB. For streaming inference, it reduces TTFT from 2.32\,s to 0.84\,s and increases decode throughput from 3.97\,tok/s to 7.52\,tok/s. Thus, a 30B-parameter MoE model can support both streaming generation and batched decoding on a single 24\,GB GPU, extending the practical benefits of \textsc{BitsMoE} to resource-constrained deployment.

\begin{table}[th]
    \centering
    \caption{Decode throughput and peak GPU memory of Qwen3-30B-A3B at 2-bit on a single RTX 3090 (24\,GB).}
    \label{tab:rtx3090_inference_efficiency}
    \tablestyle
    \begin{tabular}{@{}cccccc@{}}
        \toprule
        \multirow{2}{*}{\textbf{Batch size}}
        & \multicolumn{3}{c}{\textbf{Decode throughput (tok/s)} $\uparrow$}
        & \multicolumn{2}{c}{\textbf{Memory (GB)} $\downarrow$} \\
        \cmidrule(lr){2-4}\cmidrule(lr){5-6}
        & GPTQ & \textsc{BitsMoE} & Speedup
        & GPTQ & \textsc{BitsMoE} \\
        \midrule
        1   & 3.83  & 7.51   & \textbf{1.96$\times$} & 10.27 & 10.56 \\
        2   & 4.35  & 12.62  & \textbf{2.90$\times$} & 10.31 & 10.60 \\
        4   & 4.83  & 23.52  & \textbf{4.87$\times$} & 10.39 & 10.69 \\
        8   & 6.54  & 32.37  & \textbf{4.95$\times$} & 10.53 & 10.87 \\
        16  & 9.52  & 43.48  & \textbf{4.57$\times$} & 10.82 & 11.23 \\
        32  & 14.66 & 70.12  & \textbf{4.78$\times$} & 11.40 & 11.95 \\
        64  & 24.60 & 115.48 & \textbf{4.69$\times$} & 12.58 & 13.39 \\
        128 & 43.70 & 183.21 & \textbf{4.19$\times$} & 14.92 & 16.26 \\
        256 & 81.02 & 242.61 & \textbf{2.99$\times$} & 19.61 & 22.01 \\
        \bottomrule
    \end{tabular}
\end{table}

\subsection{Offline Quantization Time Breakdown}
\label{app:quantization_time_breakdown}

We compare offline quantization time at 2-bit on a single NVIDIA A100 PCIe (80\,GB). For all three models, GPTQ and \textsc{BitsMoE} use the same calibration set of 1,024 samples. Tab.~\ref{tab:quantization_time} reports total time for both methods and breaks down the \textsc{BitsMoE} pipeline into calibration-statistics collection (Calib.), shared-basis construction (SVD), bit allocation (ILP), quantization (Quantize), bit-plane packing (Packing), and auxiliary processing (Other). All times are in minutes, with speedups over GPTQ shown in parentheses.

\textsc{BitsMoE} requires less offline quantization time than GPTQ on all three models. It quantizes DeepSeek-V2-Lite, Qwen3-30B-A3B, and Qwen3-Next-80B-A3B in 21.32, 63.41, and 267.04 minutes, respectively, compared with GPTQ's 59, 298, and 1,279 minutes. The corresponding speedups are $2.77\times$, $4.70\times$, and $4.79\times$, demonstrating lower preprocessing cost across the evaluated model scales.

Quantization remains the most time-consuming stage of \textsc{BitsMoE}. FQCM precomputes component-wise quantization costs, enabling global bit allocation through an ILP with a linear objective and linear constraints, without repeated empirical evaluation of candidate allocations. With the same hardware and 1,024 calibration samples, \textsc{BitsMoE} quantizes Qwen3-Next-80B-A3B in 267.04 minutes, comparable to GPTQ's 298 minutes for Qwen3-30B-A3B. This comparison supports the practical scalability of our offline quantization pipeline to larger MoE models.

\begin{table}[htbp]
    \centering
    \caption{Offline quantization time (min) at 2-bit on an NVIDIA A100 PCIe (80\,GB) using 1,024 calibration samples, with per-stage costs for \textsc{BitsMoE}.}
    \label{tab:quantization_time}
    \tablestyle
    \begin{tabular}{@{}ccccccccc@{}}
        \toprule
        \multirow{2}{*}{\textbf{Model}} & \multirow{2}{*}{\textbf{GPTQ}}
        & \multicolumn{7}{c}{\textbf{\textsc{BitsMoE}}} \\
        \cmidrule(lr){3-9}
        & & \textbf{Calib.} & \textbf{SVD} & \textbf{ILP} & \textbf{Quantize} & \textbf{Packing} & \textbf{Other} & \textbf{Total} \\
        \midrule
        \textsc{DeepSeek-V2-Lite}
        & 59.00 & 6.13 & 0.99 & 0.66 & 8.30 & 0.85 & 4.39 & \textbf{21.32}\raisebox{-0.3ex}{\tiny\color{gray}~(\textbf{2.77$\times$})} \\
        \textsc{Qwen3-30B-A3B}
        & 298.00 & 11.68 & 1.93 & 4.50 & 29.90 & 3.19 & 12.21 & \textbf{63.41}\raisebox{-0.3ex}{\tiny\color{gray}~(\textbf{4.70$\times$})} \\
        \textsc{Qwen3-Next-80B-A3B}
        & 1279.00 & 18.47 & 4.17 & 57.67 & 129.27 & 12.66 & 44.80 & \textbf{267.04}\raisebox{-0.3ex}{\tiny\color{gray}~(\textbf{4.79$\times$})} \\
        \bottomrule
    \end{tabular}
\end{table}

\begin{table}[htbp]
    \centering
    \caption{Offline quantization time breakdown (min) of ILP-based methods on Qwen3-30B-A3B at 2-bit.}
    \label{tab:ilp_quantization_time}
    \tablestyle
    \begin{tabular}{@{}cccccc@{}}
        \toprule
        \textbf{Method} & \textbf{Calibration} & \textbf{ILP solving} & \textbf{Quantization} & \textbf{Other} & \textbf{Total}$\downarrow$ \\
        \midrule
        MxMoE & 2120.70 & 0.29 & 977.87 & 7.98 & 3106.83 \\
        GEMQ & 69.10 & 0.29 & 326.73 & 648.28 & 1044.40 \\
        \textsc{BitsMoE} & 11.68 & 4.50 & 29.90 & 17.33 & \textbf{63.41} \\
        \bottomrule
    \end{tabular}
\end{table}

Tab.~\ref{tab:ilp_quantization_time} compares offline quantization time on Qwen3-30B-A3B at 2-bit. All methods are evaluated on NVIDIA A100-PCIe-80\,GB GPUs using the same 1,024 C4 calibration samples, each with 2,048 tokens. MxMoE's Calibration stage jointly performs calibration, candidate quantization, and sensitivity-cost construction. GEMQ's Other stage includes the computation of gradient-weighted reconstruction errors for bit allocation.

As shown in Tab.~\ref{tab:ilp_quantization_time}, \textsc{BitsMoE} achieves $49.00\times$ and $16.47\times$ offline quantization speedups over the ILP-based methods MxMoE~\citep{duanmu2025mxmoe} and GEMQ~\citep{deng2026gemq}, respectively. Although our ILP takes 4.50 minutes compared with 0.29 minutes for both baselines, the additional solver time is more than offset by savings in cost construction and quantization. MxMoE measures output perturbations for candidate quantization schemes, while GEMQ estimates gradient-weighted reconstruction costs using calibration data. FQCM computes component-wise quantization costs directly from intrinsic spectral importance, activation-dependent importance, and bit-width-dependent distortion, avoiding repeated evaluation of candidate output errors. In addition, \textsc{BitsMoE} quantizes expert-specific spectral directions directly without GPTQ's iterative Hessian-based error compensation. GPTQ is applied only to the shared bases at 8-bit precision. These choices help explain the lower cost-construction overhead and the shorter quantization time of 29.90 minutes, compared with 977.87 minutes for MxMoE and 326.73 minutes for GEMQ.

\end{document}